\documentclass{article}

\PassOptionsToPackage{numbers, compress}{natbib}

\usepackage[main, final]{neurips_2026}

\usepackage[utf8]{inputenc} 
\usepackage[T1]{fontenc}    
\usepackage{hyperref}       
\usepackage{url}            
\usepackage{booktabs}       
\usepackage{amsfonts}       
\usepackage{nicefrac}       
\usepackage{microtype}      
\usepackage{xcolor}         
\usepackage{amssymb} 

\usepackage{amsthm}
\usepackage{siunitx}
\usepackage{graphicx}
\usepackage{algorithm}
\usepackage{algpseudocode}
\newtheorem{proposition}{Proposition}
\newtheorem{lemma}{Lemma}
\newtheorem{assumption}{Assumption}
\newtheorem{remark}{Remark}
\usepackage{amsmath}

\usepackage{hyperref}
\usepackage{cleveref}
\usepackage{etoolbox}
\definecolor{pblue}{RGB}{0, 102, 204} 
\definecolor{green}{RGB}{0, 128, 0} 
\definecolor{ahc}{RGB}{100, 128, 0} 

\crefname{assumption}{assumption}{assumptions}

\title{Generating Physically Consistent Molecules with Energy-Based Models}

\author{%
Christoph~Griesbacher\footnotemark[2]\\
Graz University of Technology\\
\texttt{griesbacher@tugraz.at}\\
\And
Lea~Bogensperger\\
University of Zurich\\
\texttt{lea.bogensperger@uzh.ch} \\
\AND
Andreas~Habring\\
Graz University of Technology\\
\texttt{andreas.habring@tugraz.at} \\
\And
Thomas~Pock\\
Graz University of Technology\\
\texttt{thomas.pock@tugraz.at} \\
}

\newcommand{\Lc}{\mathcal{L}}

\newcommand{\dd}{\mathrm{d}}

\newcommand{\R}{\mathbb{R}}
\newcommand{\dom}{\mathrm{dom}}
\newcommand{\W}{\mathcal{W}}

\begin{document}

\newpage
\maketitle
\footnotetext[2]{Corresponding author.}
\footnotetext[1]{Code available at \url{https://github.com/griesbchr/EBMol}.}
\begin{abstract}
Molecules in equilibrium follow a Boltzmann distribution, making the underlying energy landscape a physically grounded modeling objective. However, such landscapes are difficult to learn from data and, once learned, hard to sample from. Diffusion and flow-matching models sidestep these difficulties by learning a time-conditional score or transport field between noise and data, losing the energy inductive bias in exchange for a more tractable training objective. We introduce EBMol, an energy-based model (EBM) that restores this inductive bias by learning an atom-additive scalar potential without explicit simulation during training. Our method employs a flow-inspired Restoring Field Matching objective to approximate the energy landscape. We adopt the Mirror-Langevin algorithm for sampling, enabling unified updates of atomic positions and types, and incorporate parallel tempering for inference-time compute scaling. EBMol is the first EBM for 3D molecular generation to achieve state-of-the-art performance on QM9 and GEOM-Drugs. Moreover, we show that the learned energy landscape serves as a principled quality metric for ranking and filtering configurations, and demonstrate controllable generation without retraining through shape-steered sampling via potential composition and zero-shot linker design.\footnotemark[1] 

\end{abstract}

\section{Introduction}
 \label{sec:intro}

Molecular states at thermal equilibrium are governed by the Boltzmann distribution $p(\mathbf{x}) \propto \exp(-E(\mathbf{x})/(k_B T))$, in which the probability of a configuration is determined entirely by its potential energy. The standard datasets for molecular generation, QM9~\cite{ramakrishnanQuantumChemistryStructures2014} and GEOM~\cite{axelrodGEOMEnergyannotatedMolecular2022a}, are not drawn from this distribution. Instead, they contain geometry-optimized structures, each relaxed to a local energy minimum by density functional theory or semi-empirical methods. What is physically meaningful in these datasets is where the energy minima lie. The relative frequencies at which different molecular structures appear reflect enumeration strategies and computational budgets, not thermodynamic weights~\cite{reymondEnumerationChemicalSpace2012}. A scalar energy landscape whose minima coincide with stable configurations is therefore the natural modeling target: it captures exactly the physical content of the data and takes the functional form of a Boltzmann distribution, $p_\theta(\mathbf{x}) \propto \exp(-E_\theta(\mathbf{x})/\tau)$. 

Learning and sampling from such landscapes, however, have remained open problems. Training a physically-calibrated energy function requires Boltzmann-distributed data~\cite{unkeMachineLearningForce2021a} such as SPICE~\cite{eastmanSPICEDatasetDruglike2023} or MD17~\cite{chmielaMachineLearningAccurate2017}, but datasets with this property target interatomic interactions and conformational sampling rather than chemical diversity and are unsuitable for generation. Sampling is equally difficult: generating molecules from an energy function typically relies on Markov Chain Monte Carlo (MCMC) methods, which must traverse the extremely rugged, unphysical regions between noise and valid structures~\cite{elijosiusZeroShotMolecular2025a}. Diffusion~\cite{ho2020denoising,songscore} and flow matching models~\cite{lipmanflow,liuflow} sidestep both barriers by directly learning noise-to-data transport maps through time-conditional vector fields, achieving state-of-the-art results. However, the resulting models provide no explicit energy function and no mechanism to evaluate, rank, or compose configurations at inference time.

 \begin{figure}
  \centering
    \includegraphics[width=0.8\textwidth, trim={0 1.cm 0 1.5cm}, clip]{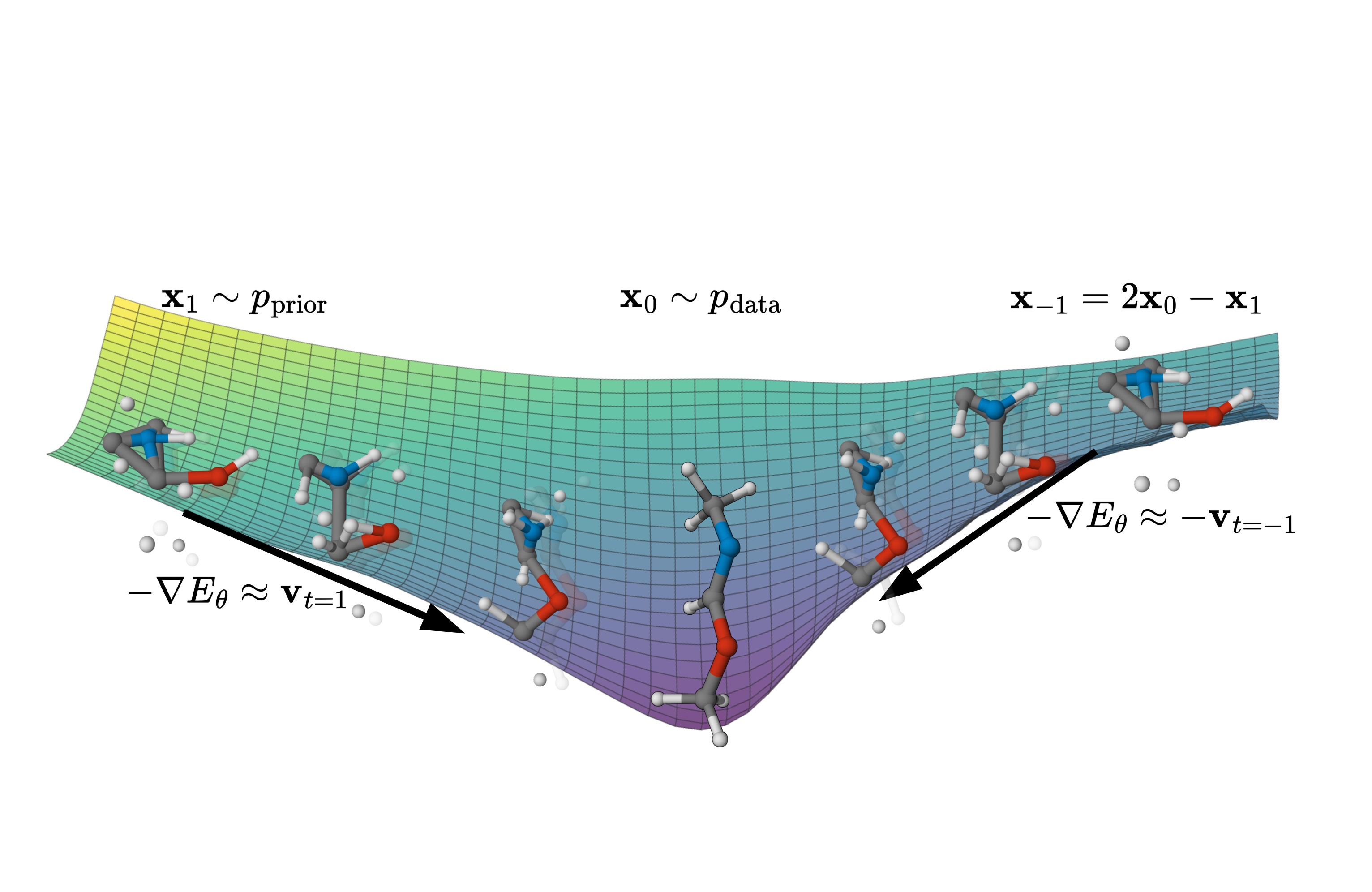}
  \caption{Overview of Restoring Field Matching. The RFM loss shapes a scalar energy landscape with data points $x_0\sim p_{\text{data}}$ as local minima by training the negative gradient $-\nabla E_{\theta}$ to point toward data from both sides: from the prior (interpolation, $t \in [0,1]$) and from beyond the data (extrapolation, $t \in [-1,0]$).}
  \label{fig: teaser fig}
\end{figure}

In this work, we introduce \textit{EBMol}, an Energy-Based Model (EBM) for Molecules that addresses both barriers. For training, we propose Restoring Field Matching (RFM, Figure~\ref{fig: teaser fig}), a flow-matching-inspired objective that shapes the energy landscape with data points as local minima, without requiring simulation or Boltzmann-distributed data. For sampling, we adapt the Mirror Langevin Algorithm (MLA)~\cite{zhangWassersteinControlMirror2020} to molecular state spaces, applying Euclidean updates on coordinates and simplex-constrained updates on atom types jointly within each step. Given that the learned energy is an explicit, persistent scalar function rather than an implicit mapping from noise to data, capabilities that time-conditional models cannot naturally provide follow directly: energy-based ranking and filtering of generated samples, composable generation through potential composition, and zero-shot conditional generation by fixing molecular fragments.

Our contributions are as follows:
\begin{itemize}

    \item We introduce Restoring Field Matching (RFM), a simulation-free     training objective for tractable learning of per-atom additive energy landscapes, and establish theoretical guarantees for the resulting optimization problem.
    
    \item We adapt the Mirror Langevin Algorithm to the mixed continuous-simplex state space of molecules, unifying coordinate and atom types updates in a single framework, combined with parallel tempering enabling inference-time compute scaling.
    
    \item EBMol is the first EBM for 3D molecular generation to achieve state-of-the-art performance on QM9 and GEOM-Drugs. This demonstrates that the benefits of energy-based modeling, including physically meaningful scoring, inference-time ranking, composable shape steering, and zero-shot linker design, are attainable without sacrificing generation quality.
\end{itemize}

\section{Related work}
\label{related work}

\paragraph{Energy-based modeling and related time-unconditional frameworks.}
Energy-based models~\cite{lecunTutorialEnergyBasedLearning} have traditionally been trained via MCMC-based maximum likelihood~\cite{hinton2002training} or score matching~\cite{songGenerativeModelingEstimating2019}, both of which sidestep the partition function but at substantial cost: MCMC is expensive and prone to mode collapse~\cite{duImprovedContrastiveDivergence2021}, while score matching either requires Hessian traces or re-introduces a time variable through noise conditioning~\cite{songGenerativeModelingEstimating2019}. Despite the recent popularity of time-conditional diffusion~\cite{ho2020denoising,songscore} and flow matching methods~\cite{lipmanflow,liuflow}, several recent works have revisited time-unconditional EBMs. Energy Matching~\cite{balcerakEnergyMatchingUnifying2025b} learns an energy landscape in two stages, combining transport and contrastive-divergence. Equilibrium Matching~\cite{wangEquilibriumMatchingGenerative2025} is conceptually closest to RFM, also enforcing stationarity at data by scaling the target field to zero, but learns an implicit non-conservative vector field rather than an explicit scalar energy. Also, neither method addresses mixed discrete-continuous state spaces. Within the molecular domain specifically, time-unconditional EBM training has previously been attempted only on highly restricted settings by \cite{jainiLearningEquivariantEnergy2021}, who train an equivariant EBM via simulation-based contrastive divergence on a single-stoichiometry subset of QM9, thus neglecting variable molecule size, mixed discrete-continuous state, or modern benchmarks.

\paragraph{Generative modeling of 3D molecules.} Various different generative modeling paradigms have been explored for 3D molecular generation. Besides autoregressive models~\cite{daigavaneSymphonySymmetryEquivariantPointCentered2023, chengScalableAutoregressive3D2025, gebauerInverseDesign3d2022a}, field-based approaches~\cite{kirchmeyerScorebased3DMolecule2024a, kirchmeyerUnifiedAllatomMolecule2025a}, and, most prominently, diffusion and flow matching models. These methods differ primarily in their treatment of bonds: bond-explicit methods use dedicated supervision~\cite{pengMolDiffAddressingAtomBond2023a, vignacMiDiMixedGraph2023a, huangLearningJoint2D2023, irwinSemlaFlowEfficient3D2025, dunnMixedContinuousCategorical2024a, reidenbachApplicationsModularCoDesign2025a}, while bond-implicit approaches, which include our work, rely on emergent geometric accuracy to recover connectivity post-hoc via OpenBabel~\cite{oboyle_open_2011, vignacMiDiMixedGraph2023a} or predefined lookup tables~\cite{hoogeboomEquivariantDiffusionMolecule2022}.
The seminal EDM~\cite{hoogeboomEquivariantDiffusionMolecule2022}  introduced equivariant diffusion via the EGNN~\cite{satorras2021n} backbone to molecular modeling. Subsequent work refined this line through latent-space variants~\cite{xuGeometricLatentDiffusion2023}, alternative architectures and training~\cite{niStraightLineDiffusionModel2025a}, flow matching~\cite{songEquivariantFlowMatching2023, kleinEquivariantFlowMatching2023, hongAccelerating3DMolecule2025}, and Bayesian flow networks~\cite{songUnifiedGenerativeModeling2023}. Several works that incorporate energy-inspired components into their generative pipeline are related to our approach. These include using Amber and MACE force fields~\cite{wuDiffusionbasedMoleculeGeneration2022a, elijosiusZeroShotMolecular2025a}, interpreting self-consistency residuals as energies in iteratively refined flow models~\cite{zhou_energy-based_2025}, or employing energy-driven objectives for conformer generation~\cite{xuEnergyGuidedFlowMatching2025a}. Crucially, however, none of these methods learn a persistent, scalar energy landscape that drives end-to-end sampling.

\paragraph{Mirror Langevin algorithm and constrained sampling.}
The Mirror Langevin Algorithm~\cite{zhangWassersteinControlMirror2020} extends mirror descent~\cite{beckMirrorDescentNonlinear2003, ben-talOrderedSubsetsMirror2001}, originally developed for constrained optimization, to sampling on constrained domains. Theoretical analyses on convex domains have established convergence rates~\cite{liMirrorLangevinAlgorithm2022, ahnEfficientConstrainedSampling} and Metropolis-adjusted extensions~\cite{srinivasanFastSamplingConstrained2024}. While MLA has recently been adapted for generative modeling on homogeneous constrained domains~\cite{liuMirrorDiffusionModels2023a}, EBMol is, to the best of our knowledge, the first unified mixed-geometry Langevin framework for joint continuous-discrete deep generative modeling.

\section{Method}
\label{sec:method} 
We propose \textit{EBMol}, a time-unconditional energy-based model that learns a scalar potential $E_{\theta}(\mathbf{x})$ over molecular configurations.  EBMol combines two largely independent components: Restoring Field Matching (RFM) for tractable shaping of the energy landscape, and the Mirror Langevin Algorithm (MLA)~\cite{zhangWassersteinControlMirror2020} for sampling from it.

\subsection{Problem formulation}
\label{sec:problem}
A molecule with $N$ atoms is represented as a joint state $\mathbf{x} = (\mathbf{c}, \mathbf{p})$, where $\mathbf{c} \in \mathbb{R}^{N \times 3}$ are atomic coordinates in Euclidean space and $\mathbf{p}=(\mathbf{p}_0, \dots, \mathbf{p}_{N-1})$ are per-atom categorical distributions with $\mathbf{p}_i \in \Delta^{K-1} = \{\mathbf{p}_i \in \mathbb{R}^K \mid p_{i,k} \ge 0, \sum_k p_{i,k} = 1\}$ over $K$ chemical elements. 
Bonds are not explicitly modeled; instead, connectivity and bond orders are inferred post-hoc from $(\mathbf{c}, \mathbf{p})$, placing the burden of chemical validity primarily on geometric consistency. We parameterize a time-unconditional energy $E_\theta(\mathbf{c}, \mathbf{p})$ over this mixed state space. To handle variable-sized molecules, we impose a permutation-invariant additive decomposition over atoms $E_\theta(\mathbf{c}, \mathbf{p}) = \sum_{i=1}^N E_\theta^{(i)}(\mathbf{c}, \mathbf{p})$.

\subsection{Learning with the Restoring Field Matching objective}
\label{sec:rfm}
 
Learning a meaningful energy landscape is challenging. Maximum-likelihood training relies on MCMC~\cite{hinton2002training}, and score matching either requires Hessian traces~\cite{hyvarinen2005estimation} or introduces a time variable~\cite{ho2020denoising,songscore}. We instead construct the landscape \emph{geometrically}: using flow-matching-style interpolation, we regress the energy gradient toward a field that points back to the data from both sides of the manifold, making data points local minima.

\paragraph{Constructing the restoring field.} Standard flow matching defines an interpolant $\mathbf{x}_t $ between $\mathbf{x}_1 \sim p_{\text{prior}}$ at $t=1$ and  $\mathbf{x}_0 \sim p_{\text{data}}$ at $t=0$, with a target velocity that is nonzero everywhere on the interval. Data is thus a trajectory endpoint, not a stationary point of any field. We begin by extending the interpolation to $t \in [-1, 1]$, so the interpolant $\mathbf{x}_t$  passes \emph{through} the data rather than terminating at it. We construct this extended interpolant separately for the two components of $\mathbf{x} = (\mathbf{c}, \mathbf{p})$. For continuous coordinates, we first align $\mathbf{c}_1$ to $\mathbf{c}_0$ using the equivariant optimal-transport coupling to reduce trajectory curvature.
 The aligned noise is denoted $\mathbf{c}_1$, and we define a linear trajectory $\mathbf{c}_t = \mathbf{c}_0 + t(\mathbf{c}_1 - \mathbf{c}_0), \quad t \in [-1, 1]$.
For categorical features, linear extrapolation leaves the simplex. We instead use a reflection, $\mathbf{p}_t = \mathbf{p}_0 + |t|(\mathbf{p}_1 - \mathbf{p}_0)$,
which remains in $\Delta^{K-1}$ and is symmetric around $t=0$.

Along each interpolant, we define a target field that points toward the data $\mathbf{x}_0$. Building on the natural choice $\mathrm{sign}(t)(\mathbf{x}_0 - \mathbf{x}_1)$, we directly encode stationarity at the data by scaling the target with a smoothing function $b: [0,1] \to [0,1]$ satisfying $b(0)=0$ and $b(1)=1$, so that the target vanishes at $t=0$:
\begin{equation}
{\mathbf{u}}^{\mathbf{c}} = b(|t|) \cdot \mathrm{sign}(t) \cdot (\mathbf{c}_0 - \mathbf{c}_1), \qquad
{\mathbf{u}}^{\mathbf{p}} = b(|t|) \cdot (\mathbf{p}_0 - \mathbf{p}_1).
\label{eq:target_smooth}
\end{equation}
The model is trained to match $-\nabla_{\mathbf{x}} E_\theta(\mathbf{x}_t)$ against this target, as described below.

\paragraph{Training objective.} 
Denoting $\mathbf{v} := -\nabla_\mathbf{x} E_\theta(\mathbf{x}_t)$
and $\mathbf{u} := (\mathbf{u}^\mathbf{c}, \mathbf{u}^\mathbf{p})$, the RFM objective is given by
\begin{equation}
    \mathcal{L}_\text{RFM} = \mathbb{E}_{\mathbf{x}_0, \mathbf{x}_1, t \sim U[-1,1]} \!\left[ \big\| \mathbf{v} - \mathbf{u} \big\|^2 \right].
\end{equation}
While the restoring target specifies the energy \emph{gradient}, it leaves the energy \emph{value} at the data
unconstrained. Without regularization, different data points may sit at
different absolute energies, biasing the sampler toward artificially
deeper basins. We therefore add a per-atom regularization term that
anchors clean-data energies toward zero:
\begin{equation}
\mathcal{L}_{\text{reg}} = \mathbb{E}_{\mathbf{x}_0 \sim p_{\text{data}}} \left[ \frac{1}{N} \sum_{i=1}^N E_\theta^{(i)}(\mathbf{x}_0)^2 \right].
\end{equation}
The full training loss is then given by 
\begin{equation}\label{eq:objective}
\mathcal{L} = \mathcal{L}_\text{RFM} + \lambda_{\text{reg}}\, \mathcal{L}_{\text{reg}},
\end{equation}
where $\lambda_{\text{reg}}$ controls the regularization strength. Training is summarized in Algorithm~\ref{alg:rfm-train}.

\subsection{Analysis}
\label{sec:analysis}
We establish two theoretical properties of the RFM objective that characterize the minimization problem defined by~\eqref{eq:objective}. The analysis considers $E$ as a general element of a function space rather than a specific network architecture. For sufficiently expressive networks, these properties are expected to hold approximately in practice, which our experiments in Section~\ref{sec: experiments} support. Formal assumptions and complete proofs for both properties are given in Appendix~\ref{sec:appendix_proofs}.

The first result ensures that the training problem is well-posed.

\begin{proposition}[Existence of solutions, informal]\label{prop:existence_informal}
    Under~appropriate assumptions on $p_\text{prior}$ and $p_\text{data}$, for every $\lambda_{\text{reg}}$ the following problem admits a unique solution:
    \[
        \min_{E} \Lc(E).
    \]
\end{proposition}
Given existence, we characterize the solution's behavior at the
two extremes of the regularization parameter. 
\begin{proposition}[Convergence for small and large regularization parameters, informal]\label{prop:characterize_solutions}
    Let $E^*(\mathbf x,\lambda)$ be the minimizer of $\Lc$ for $\lambda_{\text{reg}} = \lambda$. The following holds true:
    \begin{enumerate}
    \item[(i)] \underline{$\lambda=\infty$:} For every sequence $\lambda_n\rightarrow \infty$, there exists a subsequence $(\lambda_{n_k})_k$ such that $E^*(\cdot,\lambda_{n_k})\rightarrow 0$ $p_0$-a.e.; figuratively: $E^*(\mathbf x_0,\lambda)\approx 0$ for $\lambda\gg 0$ and $\mathbf x_0\sim p_{\text{data}}$.
    \item[(ii)] \underline{$\lambda=0$:}
    For every sequence $\lambda_n\rightarrow 0$, there exists a subsequence $(\lambda_{n_k})_k$ such that $\nabla E^*(\cdot,\lambda_{n_k})$ converges to the projection of the prescribed restoring field $\mathbf{u}$ onto the space of all conservative vector fields. That is, $\nabla E^*(\cdot,\lambda_{n_k})$ converges to the minimizer of the functional
    \begin{equation}
        E\mapsto \mathbb{E}_{\mathbf{x}_0, \mathbf{x}_1, t \sim U[-1,1]}[\bigl\| -\nabla E(\mathbf x_t) - \mathbf{u} \bigr\|^2]
    \end{equation}
    Moreover, if $\mathbf x\mapsto \mathbb{E}[\mathbf{u}|x_t=x]$ is conservative
    \begin{equation}
        \lim_{k\rightarrow\infty }\nabla E^*(\mathbf x,\lambda_{n_k}) = -\mathbb{E}[\mathbf{u}|x_t=x] \quad \text{$p_x$-a.e.}.
    \end{equation}
\end{enumerate}
\end{proposition}

The above results confirm our modeling approach theoretically: strong regularization anchors the energy to zero at the data, while the gradient matching term leads to the L2 projection of the target restoring field onto the space of conservative vector fields. At any finite $\lambda_{\text{reg}}$, the minimizer interpolates between these two extremes. However, the central modeling claim, that data points are local minima of $E_\theta$, is not established by the propositions above. It rests on the target construction ($b(0)=0$ encodes zero gradient at $t=0$) and is verified empirically in Appendix~\ref{sec:landscape_verification}.

\subsection{Sampling with annealed Mirror-Langevin algorithm}
\label{sec:mla}
We sample from $p_\theta(\mathbf{x}) \propto \exp(-E_\theta(\mathbf{x})/\tau)$ via Langevin dynamics. While standard Langevin suffices for coordinates, it pushes atom types outside the simplex. The Mirror Langevin Algorithm (MLA)~\cite{zhangWassersteinControlMirror2020} resolves this by performing updates in an unconstrained dual space and projecting back through a mirror map.

Let $\Phi: \mathcal{D} \to \mathbb{R}$ be a strictly convex, twice-differentiable mirror potential on the constrained domain $\mathcal{D}$ such that its gradient $\nabla\Phi$ maps $\mathcal{D}$ to an unconstrained dual space, and the inverse map $(\nabla\Phi)^{-1}$ projects back to $\mathcal{D}$. The MLA update with step size $\eta$, temperature $\tau$, and noise $\boldsymbol{\xi} \sim \mathcal{N}(\mathbf{0}, \mathbf{I})$ reads
\begin{align}
\mathbf{y}^{(k+1)} &= \nabla\Phi(\mathbf{x}^{(k)}) - \eta\,\nabla_\mathbf{x} E_\theta(\mathbf{x}^{(k)}) + \sqrt{2\eta\tau}\,\bigl(\nabla^2\Phi(\mathbf{x}^{(k)})\bigr)^{1/2}\boldsymbol{\xi}, \\
\mathbf{x}^{(k+1)} &= (\nabla\Phi)^{-1}(\mathbf{y}^{(k+1)}).
\end{align}
The factor $(\nabla^2\Phi)^{1/2}$ is the Riemannian metric correction induced by the mirror potential, ensuring the update is consistent with the underlying mirror Langevin SDE.

\paragraph{Continuous coordinates and categorical features.} 
The choice of mirror potential $\Phi$ specializes MLA to each component of $\mathbf{x} = (\mathbf{c}, \mathbf{p})$. For $\mathbf{c} \in \mathbb{R}^{N\times 3}$, the Euclidean potential $\Phi^\mathbf{c}(\mathbf{c}) = \tfrac{1}{2}\|\mathbf{c}\|^2$ reduces MLA to the Unadjusted Langevin Algorithm. 
For atom types $\mathbf{p} \in \Delta^{K-1}$, the negative entropy $\Phi^\mathbf{p}(\mathbf{p}) = \sum_i p_i \log p_i$ maps the simplex to logit space via $\nabla\Phi^\mathbf{p} = \log\mathbf{p} + \mathbf{1}$, with $(\nabla\Phi^\mathbf{p})^{-1} = \mathrm{softmax}$ projecting back exactly. Since the negative-entropy mirror map does not satisfy the self-concordance condition of~\cite{zhangWassersteinControlMirror2020} near the simplex boundary, we clamp $\mathbf{p} \geq \epsilon$ before each update (see Appendix~\ref{sec:boundary_stabilization} for details). The full update is given in Algorithm~\ref{alg:mla-step}.

\begin{figure}[h]
\begin{minipage}[t]{0.48\textwidth}
\begin{algorithm}[H]
\caption{RFM training step}
\label{alg:rfm-train}
\begin{algorithmic}[1]
\Require Energy $E_\theta$, smoothing function $b(t)$
\State Sample $\mathbf{x}_0 = (\mathbf{c}_0,\mathbf{p}_0) \sim p_\text{data}$
\State Sample $\mathbf{x}_1 = (\mathbf{c}_1,\mathbf{p}_1) \sim p_\text{prior}$
\State Sample $t\sim \mathcal{U}(-1,1)$
\State $T_\text{OT} \gets \mathrm{OTSolver}(\mathbf{c}_1, \mathbf{c}_0)$
\State $\mathbf{c}_t \gets \mathbf{c}_0 + t\,(T_\text{OT}\,\mathbf{c}_1 - \mathbf{c}_0)$
\State $\mathbf{p}_t \gets \mathbf{p}_0 + |t|\,(\mathbf{p}_1 - \mathbf{p}_0)$
\State $\mathbf{v} \gets -\nabla_{\mathbf{x}} E_\theta(\mathbf{c}_t, \mathbf{p}_t)$
\State $\mathbf{u}_\mathbf{c} \gets b(|t|)\,\mathrm{sign}(t)\,(\mathbf{c}_0 - \mathbf{c}_1)$
\State $\mathbf{u}_\mathbf{p} \gets b(|t|)\,(\mathbf{p}_0 - \mathbf{p}_1)$
\State $\mathcal{L}_\text{reg} \gets \frac{\lambda_{\text{reg}}}{N} \sum_{i=1}^N E_\theta^{(i)}(\mathbf{x}_0)^2$
\State $\mathcal{L} \gets \lVert \mathbf{v} - \mathbf{u} \rVert^2 + \mathcal{L}_\text{reg}$
\State \textbf{Update} $\theta \gets \theta - \eta\,\nabla_\theta \mathcal{L}$
\end{algorithmic}
\end{algorithm}
\end{minipage}
\hfill
\begin{minipage}[t]{0.48\textwidth}
\begin{algorithm}[H]
\caption{Mirror-Langevin sampling step}
\label{alg:mla-step}
\begin{algorithmic}[1]
\Require State $(\mathbf{c}, \mathbf{p})$, $E_\theta$, step size $\eta$, temperature $\tau$, noise scales $\sigma_{\mathbf{c}}, \sigma_{\mathbf{p}}$,
         clipping $\epsilon$
\State $\mathbf{g}_{\mathbf{c}} \gets \nabla_{\mathbf{c}} E_\theta(\mathbf{c}, \mathbf{p}), \quad \mathbf{g}_{\mathbf{p}} \gets \nabla_{\mathbf{p}} E_\theta(\mathbf{c}, \mathbf{p})$
\Statex \textbf{Continuous update:}
\State $\boldsymbol{\xi}_{\mathbf{c}} \sim \mathcal{N}(\mathbf{0}, \sigma_{\mathbf{c}}^2 \mathbf{I})$
\State $\mathbf{c}' \gets \mathbf{c} - \eta\, \mathbf{g}_{\mathbf{c}} + \sqrt{2\eta\tau}\; \boldsymbol{\xi}_{\mathbf{c}}$
\State $\mathbf{c}' \gets \mathbf{c}' - \frac{1}{N}\sum_{i} \mathbf{c}'_i$
\Statex \textbf{Categorical update:}
\State $\tilde{\mathbf{p}} \gets \max(\mathbf{p}, \epsilon)$
\State $\boldsymbol{\xi}_{\mathbf{p}} \sim \mathcal{N}(\mathbf{0}, \sigma_{\mathbf{p}}^2 \mathbf{I})$
\State $\mathbf{y} \gets \log \tilde{\mathbf{p}} - \eta\, \mathbf{g}_{\mathbf{p}} + \sqrt{2\eta\tau}\;\tilde{\mathbf{p}}^{-1/2} \odot \boldsymbol{\xi}_{\mathbf{p}}$
\State $\mathbf{p}' \gets \mathrm{Softmax}(\mathbf{y})$
\State \Return $(\mathbf{c}', \mathbf{p}')$
\end{algorithmic}
\end{algorithm}
\end{minipage}
\end{figure}
\subsection{Parallel tempering}
\label{sec:pt}

We run $M$ parallel chains at a fixed ladder of temperatures $\tau_1 > \tau_2 > \dots > \tau_M$, with periodic Metropolis-Hastings swaps between adjacent levels according to energies $E$ with acceptance probability
\begin{equation}
\min\!\left(1,\, \exp\!\left[\bigl(E(\mathbf{x}_j) - E(\mathbf{x}_i)\bigr)\bigl(\tfrac{1}{\tau_i} - \tfrac{1}{\tau_j}\bigr)\right]\right).
\end{equation}
Hot chains explore broadly, cold chains refine, and swaps propagate discoveries down the ladder. Periodically, we harvest samples from the coldest chain and replace them with prior samples at the same coldest level, which are swapped upward by the energy-based exchange mechanism. Subsequently, harvested samples undergo a short $\tau = 0$ relaxation. Parallel tempering offers a built-in mechanism for inference-time compute scaling by increasing the number of tempering rounds between harvests. Additional swaps and exploration before samples are extracted yields
progressively higher-quality molecules at the cost of longer sampling.

\section{Experiments}
\label{sec: experiments}
\subsection{Setup}
\label{sec: setup}

\paragraph{Datasets.}
We evaluate EBMol on the two standard benchmarks for unconditional 3D molecular generation, QM9~\cite{ramakrishnanQuantumChemistryStructures2014} and GEOM-Drugs~\cite{axelrodGEOMEnergyannotatedMolecular2022a}. QM9 contains 134k small organic molecules with up to nine heavy atoms across five element types (H, C, N, O, F), exhaustively enumerating all stable molecules within this space at DFT level. GEOM-Drugs poses a more challenging benchmark, comprising roughly 430k drug-like molecules with an average size of 44 atoms (up to 181) across 15 element types. This dataset is substantially more diverse than QM9. 

\paragraph{Implementation.}
\label{sec:implementation}
To ensure SE(3) invariance of the energy, we employ a standard EGNN~\cite{satorras2021n} network, propagating equivariant messages across a fully connected atomic graph. We compute the per-atom energies 
by applying a two-layer MLP to the invariant node features of the final layer, omitting the final equivariant feature update. The total molecule energy is calculated by summing over the node energies, preserving SE(3) invariance of $E_\theta$. During training, we use a diagonal anisotropic Gaussian position prior whose per-molecule covariance is set to the principal-component eigenvalues of a training molecule with the same atom count, introducing a mild data-dependent shape bias that improves training stability and performance. For atom types, we use a Dirichlet distribution with concentration $1/K$, where $K$ is the number of atom types. We employ $\mathrm{tanh}(\cdot)$ as smoothing function. Pseudocode for the training and sampling algorithms is presented in Algorithms~\ref{alg:rfm-train} and~\ref{alg:mla-step}. %
During generation, we sample the atom count $N$ from the empirical training distribution and base the parallel tempering swaps on the largest per-atom energy $\max_iE_\theta^{(i)}(\mathbf{x})$. A full list of hyperparameters, more experiment details and ablations on design choices can be found in Appendices \ref{sec: model sampler ablation} and \ref{sec: hyperparameters and details}. 

\paragraph{Baselines and evaluation.} 
We compare against a variety of established bond-implicit baselines for molecular generation~\cite{hoogeboomEquivariantDiffusionMolecule2022, wuDiffusionbasedMoleculeGeneration2022a, songEquivariantFlowMatching2023, cornetEquivariantNeuralDiffusion2024a, xuGeometricLatentDiffusion2023, hongAccelerating3DMolecule2025, songUnifiedGenerativeModeling2023, niStraightLineDiffusionModel2025a}. All baselines share the same EGNN network architecture~\cite{satorras2021n}, with the exception of SLDM, employing a modified EGNN variant based on UniGEM~\cite{fengUniGEMUnifiedApproach2024} with more parameters (2.2M vs 4.2M for GEOM-Drugs). On GEOM-Drugs, we restrict baselines to methods with publicly available checkpoints to ensure consistent re-evaluation under our revised evaluation protocol.  

On QM9 we follow the standard protocol of~\cite{hoogeboomEquivariantDiffusionMolecule2022}, reporting atom and molecule stability, validity, uniqueness, and novelty. On GEOM-Drugs, recent models have begun to saturate this protocol: its QM9-calibrated bond and valency lookup tables yield a data baseline of only 86.5\% atom stability, a ceiling already exceeded by SLDM~\cite{niStraightLineDiffusionModel2025a}. To restore this benchmark's discriminative power, we adopt a revised protocol. We perform bond inference via OpenBabel~\cite{oboyle_open_2011}, which additionally infers formal charges and aromaticity. This additional information unlocks the usage of the valency tables proposed by~\cite{nikitinGEOMdrugsRevisitedMore2025}, which are specifically calibrated for the element and bond-type diversity encountered in GEOM-Drugs. Furthermore, we use the OpenBabel-inferred molecules for RDKit sanitization with stereochemistry assignment for more expressive validity, uniqueness, and novelty metrics. Under this revised protocol, the data baseline rises to 99.6\% atom stability (from 86.5\%) and 88.2\% molecule stability (previously \textasciitilde0\%), indicating a much more faithful assessment of chemical validity. We additionally report the Vendi score~\cite{friedmanVendiScoreDiversity2023} for molecular diversity and the physical metric suite of~\cite{nikitinGEOMdrugsRevisitedMore2025}, which assesses differences in bond lengths, bond angles and torsion angles between generated molecules and their GFN2-xTB-optimized~\cite{bannwarthGFN2xTBAnAccurateBroadly2019} counterparts, alongside the associated mean and median relaxation energies.

\subsection{Unconditional generation}
\label{sec: unconditional generation}

\begin{table}[ht]
  \centering
   \footnotesize 
  \caption{Standard molecular generation metrics on QM9, computed on 10k generated samples per model. Results marked with * are from our own experiments. Unless otherwise noted in the subscript, models use 1000 NFEs.}
  \label{tab:qm9_results}
  \resizebox{\textwidth}{!}{ 
  \begin{tabular}{l c c c c c c}
    \toprule
    \textbf{Method} & \textbf{Atom $\uparrow$} & \textbf{Molecule $\uparrow$} & \textbf{Validity $\uparrow$} & \textbf{Uniqueness $\uparrow$} & \textbf{Valid \&} & \textbf{Novelty $\uparrow$} \\
    & \textbf{Stability} & \textbf{Stability} & & & \textbf{Unique $\uparrow$} & \\
    \midrule
    Data          & 99.40 & 95.30 & 97.70 & 99.97 & 97.60 & --    \\
    \midrule
    EDM           & 98.70 & 82.00 & 91.90 & --    & 90.70 & 65.70 \\
    EDM-Bridge    & 98.80 & 84.60 & 92.00 & --    & 90.70 & --    \\
    EquiFM$_{210}$        & 98.90 & 88.30 & 94.70 & --    & {93.50} & --    \\
    END           & 98.90 & 89.10 & 94.80 & --    & 92.60 & --    \\
    GEOLDM        & 98.90 & 89.40 & 93.80 & --    & 92.70 & 57.00* \\
    GOAT          & 99.20 & --    & 92.90 & 99.00 & --    & \textbf{78.60} \\
    GeoBFN$_{2000}$      & 99.31 & 93.32 & 96.88 & --    & \textbf{92.41} & 65.30 \\
    SLDM          & 99.43 & 95.42 & 97.07 & --    & 90.42 & --    \\
    SLDM*         & 99.07 & 92.32 & 95.53 & {95.28} & {91.02} & 42.46    \\
    \midrule
    EBMol$_{930}$ &  99.35 & 92.34 & 96.88 & 86.61 & 83.91 & 45.68 \\
    EBMol$_{1370}$ & 99.65 & 95.58 & 98.32 & 82.78 & 81.39 & 43.37 \\
    EBMol$_{1810}$ &  \textbf{99.76} & \textbf{97.09} & \textbf{98.93} & 82.30 & 81.42 & 42.08 \\
    \bottomrule
  \end{tabular}
  }
\end{table}
The results for unconditional generation on QM9 are summarized in Table~\ref{tab:qm9_results}. We benchmark EBMol for three different compute-settings, where the lowest setting of 930 network function evaluations (NFEs) matches the baseline models. The results show that EBMol$_{930}$ is competitive with the currently best bond-implicit model SLDM, while using approximately half the parameters. 

When scaling compute to 1810 NFEs, EBMol exceeds all baselines in stability and validity. This includes the strong baseline GeoBFN that also achieves high quality samples through inference-time compute adaptation. Uniqueness and novelty are lower, at ~82\% and ~42\%, respectively. This is expected on QM9 because the dataset exhaustively enumerates its composition space~\cite{reymondEnumerationChemicalSpace2012}, so higher-quality molecules are more likely to already appear in the training set and novel compositions tend to violate stability constraints~\cite{hoogeboomEquivariantDiffusionMolecule2022, wuDiffusionbasedMoleculeGeneration2022a}.
The quality--diversity tradeoff observed in Tables~\ref{tab:qm9_results} and \ref{tab:geom_results} is a direct consequence of the parallel tempering sampler concentrating probability mass on the lowest-energy modes. While this ensures high chemical validity, it naturally limits diversity.

Shifting towards the more challenging GEOM-Drugs benchmark, we report the revised metrics in Table~\ref{tab:geom_results}. At a compute-matched setting, EBMol already surpasses the SLDM baseline, a performance gap that widens substantially as compute scales. While we observe a slight decrease in uniqueness as stability improves, this shift is consistent with the expected quality--diversity tradeoff. The model still produces 96.4\% unique samples among 10k generations, suggesting a slow decay in diversity rather than a catastrophic mode collapse.
Randomly selected samples are visualized in Appendix~\ref{sec: generated samples} and a failure mode analysis can be found in Appendix~\ref{app:failure_modes}.

The physical metrics presented in Table~\ref{tab:geom_drugs_energy} clearly demonstrate the impact of the energy inductive bias. In the medium-compute configuration, EBMol$_{1960}$ already outperforms bond-implicit baselines and is the first such model to surpass the MMFF~\cite{halgrenMerckMolecularForce1996} force field. When compute is scaled, EBMol$_{7240}$ achieves a median energy decrease of 1.78 kcal/mol, an order of magnitude lower than previous bond-implicit models. Notably, EBMol remains competitive with the bond-explicit Megalodon (3.19 kcal/mol, 40.6M parameters), despite having 18$\times$ fewer parameters and no explicit bond supervision. This trend extends to bond length and angle distributions, where EBMol$_{7240}$ matches or exceeds Megalodon's performance.

\begin{table}[ht]
  \centering
   \footnotesize 
    \caption{Revised stability, validity, uniqueness, novelty, and diversity metrics on GEOM-Drugs, computed on 10k generated samples per model. Compute budgets (NFEs) are indicated in each method's subscript.}
  \label{tab:geom_results}
  \resizebox{\textwidth}{!}{
  \begin{tabular}{l c c c c c c}
    \toprule
    \textbf{Method} & \textbf{Atom $\uparrow$} & \textbf{Molecule $\uparrow$} & \textbf{Validity $\uparrow$} & \textbf{Uniqueness $\uparrow$} & \textbf{Novelty $\uparrow$}& \textbf{Diversity $\uparrow$} \\
    & \textbf{Stability} & \textbf{Stability} & & & & \\
    \midrule
    Data            & 99.60 & 88.62 & 97.91 & 97.89 & -- &  901.7  \\
    Noise           & 37.27 & 00.00 & 71.89 & 100.00 & 100.00 & 1873.8 \\ 
    \midrule
    EDM$_{1000}$     & 87.07 & 1.03  & 84.43 & \textbf{99.99} & 99.94 &  \textbf{1621.6}  \\
    GEOLDM$_{1000}$ & 91.61 & 6.08  & 86.93 & 99.96 & 99.88 &  1582.5  \\
    SLDM$_{1000}$   & 95.27 & 20.59 & 86.98 & 99.95 & 99.93 &  1474.0 \\
    \midrule
    EBMol$_{1080}$    & 95.86 & 24.68 & 95.36 & 99.67 & 99.69 & 1096.8 \\
    EBMol$_{1960}$   & 97.17 & 37.40 & 97.02 & 99.39 & 99.50 & 927.0 \\
    EBMol$_{3720}$   & 98.15 & 52.31 & 98.61 & 98.47 & 99.42 & 783.0\\
    EBMol$_{7240}$   & \textbf{98.68} & \textbf{61.42} & \textbf{99.03} & 96.42 & 99.34 & 645.0\\
    \bottomrule
  \end{tabular}
  }
\end{table}

\begin{table}[ht]
  \centering
    \caption{Structural and energy metrics on GEOM-Drugs, computed on valid and connected molecules only. Bond lengths (\AA), angles (degrees), energies (kcal/mol). The best bond-explicit models from GEOM-Drugs Revisited~\cite{nikitinGEOMdrugsRevisitedMore2025} are included for reference.}
  \label{tab:geom_drugs_energy}
  \resizebox{\textwidth}{!}{ 
  \begin{tabular}{l c c c c c c}
    \toprule
    \textbf{Method} & \textbf{Bond Lengths $\downarrow$} & \textbf{Bond Angles $\downarrow$} & \textbf{Torsions $\downarrow$} & \textbf{Avg Energy $\downarrow$} & \textbf{Med Energy $\downarrow$}  \\
    & \textbf{($\times 10^{-2}$)} & & & \textbf{Diff.} & \textbf{Diff.} \\
    \midrule
    Data                     & $0.00 ^{\pm 0.00}$ & $0.00 ^{\pm 0.00}$ & $0.00 ^{\pm 0.01}$ & $0.00 ^{\pm 0.01}$ & $0.00 ^{\pm 0.00}$  \\
    MMFF $\rightarrow$ GFN2-xTB & $1.12 ^{\pm 0.01}$ & $1.22 ^{\pm 0.00}$ & $4.89 ^{\pm 0.10}$ & $11.40 ^{\pm 0.20}$ & $9.84 ^{\pm 0.06}$  \\

    FlowMol2$_{100}$ (4.3M)          & $1.30 ^{\pm 0.04}$ & $1.62 ^{\pm 0.02}$ & $15.00 ^{\pm 0.30}$ & $24.30 ^{\pm 0.80}$ & $17.90 ^{\pm 0.50}$\\

    Megalodon$_{500}$ (40.6M)          & $0.66 ^{\pm 0.02}$ & $0.71 ^{\pm 0.01}$ & $5.58 ^{\pm 0.11}$ & $5.76 ^{\pm 0.27}$ & $3.19 ^{\pm 0.12}$\\
    \midrule
    EDM$_{1000}$ (2.2M)                 & $5.26 ^{\pm 0.03}$ & $5.73 ^{\pm 0.05}$ & $30.82 ^{\pm 0.18}$ & $175.34 ^{\pm 2.44}$ & $155.58 ^{\pm 1.87}$ \\
    GEOLDM$_{1000}$ (2.2M)              & $3.20 ^{\pm 0.11}$ & $3.90 ^{\pm 0.06}$ & $24.51 ^{\pm 0.19}$ & $99.25 ^{\pm 4.49}$ & $80.13 ^{\pm 2.13}$ \\
    SLDM$_{1000}$ (4.2M)       & $1.66 ^{\pm 0.03}$ & $1.34 ^{\pm 0.02}$ & $10.07 ^{\pm 0.08}$ & $23.77 ^{\pm 0.87}$ & $17.90 ^{\pm 0.30}$ \\
    \midrule
    EBMol$_{1080}$ (2.2M)      & $1.77 ^{\pm 0.02}$ & $1.67 ^{\pm 0.03}$ & $8.46 ^{\pm 0.13}$ & $40.77 ^{\pm 9.16}$ & $6.67 ^{\pm 0.30}$ \\
    EBMol$_{1960}$ (2.2M)      & $1.22 ^{\pm 0.04}$ & $1.19 ^{\pm 0.04}$ & $6.53 ^{\pm 0.19}$ & $26.13 ^{\pm 7.87}$ & $4.12 ^{\pm 0.07}$ \\
    EBMol$_{3720}$ (2.2M)      & $0.83 ^{\pm 0.01}$ & $0.81 ^{\pm 0.01}$ & $4.56 ^{\pm 0.07}$ & $10.36 ^{\pm 1.55}$ & $2.35 ^{\pm 0.05}$ \\
    EBMol$_{7240}$ (2.2M)      & $\mathbf{0.65 ^{\pm 0.01}}$ & $\mathbf{0.62 ^{\pm 0.01}}$ & $\mathbf{3.65 ^{\pm 0.06}}$ & $\mathbf{5.27 ^{\pm 0.39}}$ & $\mathbf{1.78 ^{\pm 0.04}}$ \\
    \bottomrule
  \end{tabular}
  }
\end{table}

\subsection{Composable and conditional generation}
\label{sec: conditional and composable}

\begin{figure}
    \centering
    \includegraphics[width=1\linewidth, trim={0 0.2cm 0 0.1cm}, clip]{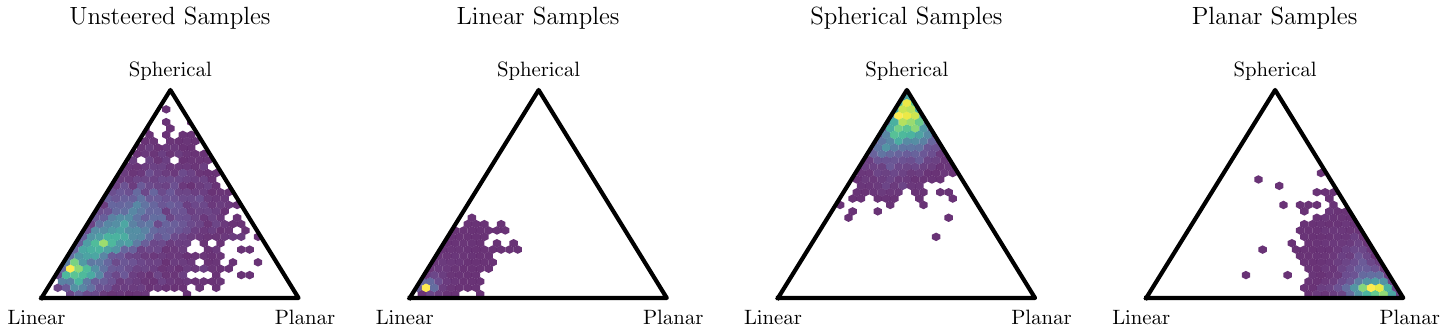}
    \caption{Shape distribution of 5000 generated molecules. Left: unsteered generation. Right: generation with composable shape potentials targeting linear, planar, and spherical geometries. Steering successfully redirects sampling toward underrepresented shapes.}
    \label{fig:shape steering}
\end{figure}

\paragraph{Shape-steered generation.} We compose the learned potential $E_\theta$ with a differentiable potential $U$ representing molecular shapes. This potential is defined via PCA eigenvalues of atom positions and targets either linear ($U_{\text{lin}}=\lambda_2 + \lambda_3$), planar ($U_{\text{disk}}=\lambda_3 + \lambda_1 - \lambda_2$) or spherical ($U_{\text{sphere}}=\lambda_1 - \lambda_3$) shapes during generation. Here $\lambda_1 \ge \lambda_2 \ge \lambda_3$ are the PCA eigenvalues of the covariance matrix of $\mathbf{c}$, measuring the stretch along each principal axis. For example, minimizing $U_\text{lin}= \lambda_2 + \lambda_3$ suppresses the two minor axes, thus collapsing the molecule onto a single axis.
Steered molecules successfully shift toward the target shape (Figure~\ref{fig:shape steering}) while maintaining reasonable stability and validity (Table~\ref{tab:conditional}), confirming that the learned energy landscape contains valid basins even for shapes underrepresented in the training data. 

\paragraph{Case study on linker design.} The additive energy composition naturally enables conditioning on arbitrary inputs, without any retraining. To test this, we conduct a linker design case study where we fix a subset of atoms and sample the remaining ones. We evaluate this zero-shot linker design by removing linking fragments between stable parts of molecules like ring structures and replacing the removed atoms with Gaussian noise matched to the removed atoms spatial statistics. We subsequently run a modified parallel tempering sampler for 2000 iterations. Following the evaluation protocol of~\cite{igashovEquivariant3DconditionalDiffusion2024}, EBMol achieves an average success rate of 73.9\% across six conditioning tasks of varying complexity (Table~\ref{tab:linker_design}), all without task-specific training.

\begin{table}[t]
\centering
\caption{Zero-shot conditional generation results. \textbf{Left:} Linker design across six problems of varying complexity, each evaluated over 100 attempts. Success requires a single connected component passing RDKit sanitization. \textbf{Right:} Shape-steered generation across three target shape classes alongside unsteered generation, each over 5000 samples. No task-specific training is required.} 
\label{tab:conditional}
\begin{minipage}[t]{0.50\textwidth}
\centering
\textit{Linker design}\\[2pt]
\resizebox{\textwidth}{!}{
\begin{tabular}{l c c c}
\toprule
\textbf{Problem} & \textbf{Free} & \textbf{Succ.\ $\uparrow$} & \textbf{Uniq.\ $\uparrow
$} \\
&\textbf{Atoms}&&\\
\midrule
P1 (easy)         & 5         & 95.2 & 16.4 \\
P2 (medium)         & 7         & 77.9 & 48.1 \\
P3 (medium)       & 12        & 88.5 & 84.5 \\
P4 (hard, 3 frags)& 10        & 55.8 & 51.5 \\
P5 (hard, 2 locs) & 5 + 5     & 75.0 & 83.3 \\
P6 (hard, 2 locs) & 8 + 6     & 51.0 & 93.3 \\
\midrule
Average           &           & 73.9 & 62.85\\
\bottomrule
\end{tabular}
}
\label{tab:linker_design}
\end{minipage}
\hfill
\begin{minipage}[t]{0.49\textwidth}
\centering
\textit{Shape-steered generation}\\[2pt]
\resizebox{\textwidth}{!}{
\begin{tabular}{l c c c}
\toprule
\textbf{Target} & \textbf{Mol.\ } & \textbf{Bond } & \textbf{Validity $\uparrow$} \\
&\textbf{Stab.\ $\uparrow$}& \textbf{Stab.\ $\uparrow$}&\\
\midrule
Linear     & 20.76 & 95.57 & 88.28 \\
Planar     &  8.96 & 92.90 & 88.80 \\
Spherical  & 31.66 & 96.59 & 96.70 \\
\midrule
Unsteered  & 37.52 & 97.09 & 96.68 \\
\bottomrule
\end{tabular}
}
\label{tab:shape_steering}
\end{minipage}
\end{table}

\subsection{Physical meaningfulness of the learned energy}
\label{sec: physically meaningful}
The energy inductive bias not only improves generation quality but produces a physically meaningful scoring function. We evaluate the correlation between EBMol energy and the energy decrease upon GFN2-xTB~\cite{bannwarthGFN2xTBAnAccurateBroadly2019} relaxation for molecules generated by the independently trained EDM (Figure~\ref{fig: energy correlation} in Appendix~\ref{sec:energy_correlation_appendix}). The correlation ($r = 0.65, p<0.0001$) emerges despite EBMol never receiving any energy-related supervision: the model is trained solely on atomic positions and atom types, yet learns to assign high energy to physically strained geometries. As a result, $E_\theta$ can serve as a lightweight first-pass filter for screening generated molecules by physical quality, providing order-of-magnitude speedups over GFN2-xTB optimization when screening large molecular batches (Appendix~\ref{app:first_pass_filter}).

\section{Conclusion and limitations}
\label{sec:conclusion}
 
\paragraph{Contributions.} EBMol demonstrates the potential of energy-based models to serve as a versatile and practical framework for 3D molecular generation, achieving state-of-the-art results on QM9 and GEOM-Drugs while providing critical capabilities that standard distribution-matching frameworks lack: energy-based ranking, steering via potential composition, and conditional generation from a single model. The core components are Restoring Field Matching, which shapes a scalar energy landscape without simulation, and Mirror-Langevin sampling, unifying continuous and discrete updates within a principled framework. Perhaps most notably, EBMol generates structures surprisingly close to physical local minima: its median relaxation energy of 1.78 kcal/mol is an order of magnitude below previous bond-implicit methods. This performance is competitive with bond-explicit models at 18$\times$ the parameter count, despite EBMol requiring no bond or energy supervision.

\paragraph{Limitations.} While we establish the design space for learning an energy landscape combined with Mirror-Langevin sampling, we do not systematically explore advanced sampler designs beyond the three algorithms in our sampler ablation (Table~\ref{tab:sampling_ablation}). Currently, achieving the best results requires a significant number of NFEs, and the employed sampling algorithms are sensitive to hyperparameter choices. On the theoretical side, our analysis establishes well-posedness of the training problem and characterizes the solution at the regularization extremes, but the claim that data points correspond to local minima in the learned landscape is motivated by the target construction and verified only empirically.
 
\paragraph{Outlook.}  A natural extension to EBMol is integrating energy calibration by replacing the zero-anchor objective with a loss that aligns $E_{\theta}(\mathbf{x}_0)$ to a physical quantity such as DFT or GFN2-xTB energy values. A calibrated energy would also enable coupled generative-simulation workflows in which the same $E_\theta$ both generates candidate molecules and serves as an approximate molecular dynamics potential. On the sampling side, we expect significant improvements from more sophisticated samplers such as accelerated Langevin methods~\cite{falkInertialLangevinAlgorithm2025, leimkuhlerContractionConvergenceRates2024} for faster sampling or from repulsive interaction potentials~\cite{balcerakEnergyMatchingUnifying2025b} for improved sample diversity.

\section{Acknowledgements}
The financial support by the LEAD project DigiBioTech of Graz University of Technology is gratefully acknowledged.
\bibliographystyle{unsrt}
{\small
    \bibliography{references}
}
\appendix
\section{Proofs}\label{sec:appendix_proofs}

In this section we provide rigorous formulations and proofs of the theoretical results in~\cref{sec:analysis}. Let us first fix some notation.
We denote the marginal densities of $\mathbf x_t$ by $p_t = \mathrm{law}(\mathbf x_t)$ and $p_x(\mathbf x) = \int p_t(\mathbf x)\, p(t)\,\mathrm{d}t$. We denote the domain of $\mathbf x$ as $\Omega = \mathbb{R}^{N \times 3} \times \Delta^{K-1}$. We define the Sobolev space $W^{1,2}(\Omega) = \{f\in L^2(\Omega)\,|\, \nabla f\in L^2(\Omega)\}$ where the gradient is understood in the usual weak sense.
Moreover, as the underlying solution space for our training problem we define the weighted Sobolev space
\begin{equation}
    \mathcal{W}^{1,2}(\Omega) = \Bigl\{
        f : \Omega \to \mathbb{R}
        \;\Big|\;
        \nabla f \text{ exists weakly and }
        \int_\Omega \bigl(|f(\mathbf x)|^2 + |\nabla f(\mathbf x)|^2\bigr)\,
        p_x(\mathbf x)\,\mathrm{d}\mathbf x < \infty
    \Bigr\}.
\end{equation}
We impose the following assumptions on the occurring distributions, respectively their densities.

\begin{assumption}\label{ass}\
    \begin{enumerate}
        \item The domain $\dom(p_x)$ is connected, $p_x$ is continuous, bounded above and strictly positive on $\Omega$. Moreover, $p_x$ satisfies a Poincar\'e inequality, that is, there exists $C>0$ such that 
        \begin{equation}
            \|f - \bar f\|_{L^2(p_x)}\leq C\|\nabla f\|_{L^2(p_x)},\quad f\in\W^{1,2}(\Omega)\tag{PI}
        \end{equation}
        where $\bar f = \int f(\mathbf x) p_x(\mathbf x)\dd \mathbf x$.
        \item It holds $p_0\ll p_x$ with $\frac{p_0}{p_x}$ locally bounded. Moreover, there exists $A\subset\R^d$ such that for $\mathbf x\in A$, $p_0(\mathbf x)=0$ implies $p_x(\mathbf x)=0$ and $\int_A p_x(\mathbf x)\dd \mathbf x>0$.
    \end{enumerate}
\end{assumption}

\begin{remark}
    Let us briefly comment on~\cref{ass}. 
    \begin{enumerate}
        \item In the first assumption the most interesting part is the Poincar\'e inequality. This assumptions appears frequently in the context of diffusion processes and Langevin sampling~\cite{vempala2019rapid}. A Poincar\'e inequality is, in particular, satisfied for a density $p$, if
        \[
            \langle\nabla\log p(\mathbf x),\mathbf x\rangle\leq -a\|\mathbf x\| + b
        \]
        for some $a>0$, $b\in \R$~\cite{bakry2008simple}. That is, the Poincar\'e inequality imposes in some sense sufficiently fast decaying tails on the density.
        \item Regarding the second assumption, $p_0\ll p_x$ is to be expected, as $p_x$ is effectively a perturbed version of $p_0$. The converse absolute continuity on the set $A$ is, for instance, satisfied if there exists a set $A$ of positive measure $p_x$ with $p_0>0$ on $A$.
    \end{enumerate}
\end{remark}
Before considering the main results, we prove the following equivalence of norms in $\W^{1,2}(\Omega)$.
\begin{lemma}\label{lemma:appendix:poincare}
    Under~\cref{ass} for every $u\in \W^{1,2}(\Omega)$ we have 
    \begin{equation}
        \|u\|_{L^2(p_x)}\leq \|u\|_{L^2(p_0)} +  \|\nabla u\|_{L^2(p_x)}.
    \end{equation}
\end{lemma}
\begin{remark}
    The important aspect above is that the first L2 norm on the right-hand side is with respect to $p_0$.
\end{remark}
\begin{proof}
    Assume to the contrary, we find a sequence $(u_n)_n\subset \W^{1,2}(\Omega)$ with 
    \begin{equation}\label{eq:appendix_poincare}
        \|u_n\|_{L^2(p_x)}\geq n ( \|u_n\|_{L^2(p_0)} +  \|\nabla u_n\|_{L^2(p_x)}).
    \end{equation}
    Define
    \begin{equation}
        v_n = \frac{u_n}{\|u_n\|_{L^2(p_x)}}
    \end{equation}
    so that $\|v_n\|_{L^2(p_x)}=1$ and by~\eqref{eq:appendix_poincare}
    \begin{equation}\label{eq:appendix_poincare2}
        \frac{1}{n}\geq  ( \|v_n\|_{L^2(p_0)} +  \|\nabla v_n\|_{L^2(p_x)}).
    \end{equation}
    In particular, we can extract a subsequence $(v_{n_k})_k$ such that 
    \begin{equation}
        \begin{aligned}
            v_{n_k}&\rightharpoonup v\\
            \nabla v_{n_k}&\rightharpoonup g
        \end{aligned}
    \end{equation}
    both in $L^2(p_x)$ and by the assumed positivity and continuity of $p_x$ we find in particular, that $v_{n_k}\rightharpoonup v$ in $W^{1,2}_{\mathrm{loc}}(\Omega)$, that is, with respect to the regular Lebesgue measure, and $g=\nabla v$. Moreover, $v\in \W^{1,2}(\Omega)$.
    Now,~\eqref{eq:appendix_poincare2} implies that $\nabla v_{n_k}\rightarrow 0$ and since weak limits are unique we have $\nabla v = g = 0$. The Poincar\'e inequality, \textit{cf.}~\cref{ass}, then implies
    \[
        \|v-\bar v\|_{L^2(p_x)}\leq C\|\nabla v\|_{L^2(p_x)}=0,
    \]
    that is, $v(\mathbf x)$ is equal to a constant, say $c\in\R$, $p_x$-a.e. Similarly, ~\eqref{eq:appendix_poincare2} yields $v_{n_k}\rightarrow 0$ in $L^2(p_0)$ so that $v(\mathbf x) = 0$ for $p_0$-a.e. $\mathbf x$. Using the set $A$ from~\cref{ass} we find
    \begin{equation}
        c\int_{A}p_x(\mathbf x)\dd \mathbf x = \int_{A} v(\mathbf x)\frac{p_x(\mathbf x)}{p_0(\mathbf x)} p_0(\mathbf x)\dd \mathbf x = 0
    \end{equation}
    where the last equality follows from $v(\mathbf x) = 0$, $p_0$-a.e. Since $\int_{A}p_x(\mathbf x)\dd \mathbf x>0$ it follows $c=0$, that is, $v=0$, $p_x$-a.e. Thus, we have that $v_{n_k}\rightharpoonup 0$ in $L^2(p_x)$. Since constants are elements of $L^2(p_x)$ this also implies that
    \begin{equation}
        \bar v_{n_k} = \int v_{n_k}(\mathbf x)p_x(\mathbf x)\dd \mathbf x \rightarrow 0.
    \end{equation}
    Then, the Poincar\'e inequality implies
    \begin{equation}
        \begin{aligned}
            \|v_{n_k}\|_{L^2(p_x)} 
            \leq& \|v_{n_k} - \bar v_{n_k}\|_{L^2(p_x)} + \|\bar v_{n_k} \|_{L^2(p_x)}\\
            \leq& \|v_{n_k} - \bar v_{n_k}\|_{L^2(p_x)} + \bar v_{n_k}\\
            \leq& C \|\nabla v_{n_k} \|_{L^2(p_x)} + \bar v_{n_k}\rightarrow 0
        \end{aligned}
    \end{equation}
    We, therefore, find
    \begin{equation}
        1 = \|v_{n_k}\|_{L^2(p_x)}\rightarrow 0
    \end{equation}
    which is a contradiction.
\end{proof}

We can now prove existence and uniqueness of solutions of the training problem. Let us first rigorously reformulate the result.
\begin{proposition}[Existence of solutions]\label{thm:existence}
    Under~\cref{ass} the problem 
    \[
        \min_{E\in \W^{1,2}(\Omega)} \Lc(E)
    \]
    admits a unique solution.
\end{proposition}
\begin{proof}
    Denoting $\mathbf{u} = (\mathbf{u}^{\mathbf{c}},\mathbf{u}^{\mathbf{p}})$ we can rewrite the objective as
    \begin{equation}
        \begin{aligned}
            \Lc (E)
            = \iiint &\biggl\{\bigl\| -\nabla E_{\theta}(\mathbf{x}_t) - \mathbf{u} \bigr\|^2 + \lambda_{\text{reg}}E(\mathbf{x}_0)^2\biggr\} p(\mathbf x_0,\mathbf x_t,t) \dd \mathbf x_t\dd \mathbf x_0\dd t\\
            = \iiint &\bigl\| -\nabla E_{\theta}(\mathbf{x}_t) - \mathbf{u} \bigr\|^2 p(\mathbf x_0,\mathbf x_t,t)\dd \mathbf x_t\dd \mathbf x_0\dd t + \lambda_{\text{reg}} \int E(\mathbf x)^2 p_0(\mathbf x) \dd \mathbf x
        \end{aligned}
    \end{equation}
    Now let $(E_n)_n\subset \W^{1,2}(\Omega)$ be a minimizing sequence, that is, 
    \[
        \lim_{n\rightarrow\infty}\Lc(E_n) = \inf_{E\in \W^{1,2}(\Omega)}\Lc(E).
    \]
    By definition of $\Lc$ it follows that $\|E_n\|_{L^2(p_0)}$ and $\|\nabla E_n\|_{L^2(p_x)}$ are bounded. The Poincar\'e inequality, \textit{cf.}~\cref{lemma:appendix:poincare} then yields that also $\|E_n\|_{L^2(p_x)}$ is bounded so that we can extract a subsequence $(E_{n_k})_k$ and functions $E^*\in L^2(p_x)$, $g\in L^2(p_x)$ such that $E_{n_k}\rightharpoonup E^*$, $\nabla E_{n_k}\rightharpoonup g$ both in $L^2(p_x)$. Moreover, by the assumed properties of $p_x$ we also find that the weak convergence is in $W^{1,2}(\Omega)$ and that $\nabla E^* = g$ and $E^*\in \W^{1,2}(\Omega)$. By potentially extracting another subsequence, $E_{n_k}$ converges also weakly in $L^2(p_0)$ and the limit is again $E^*$ due to the assumed boundedness of $p_0/p_x$ as the latter implies that weak convergence in $L^2(p_x)$ includes also weak convergence in $L^2(p_0)$.
    Note that 
    \begin{equation}
        G\mapsto \iiint \bigl\| -G(\mathbf{x}_t) - \mathbf{u} \bigr\|^2 p(\mathbf x_0,\mathbf x_t,t)\dd \mathbf x_t\dd \mathbf x_0\dd t.
    \end{equation}
    is convex and continuous in $L^2(p_x)$ and therefore also lower semi-continuous with respect to weak convergence in $L^2(p_x)$. Similarly,
    \begin{equation}
        E\mapsto \iiint \lambda_{\text{reg}}E(\mathbf{x}_0)^2 p(\mathbf x_0)\dd \mathbf x_0
    \end{equation}
    is weakly lower semi-continuous in $L^2(p_0)$. Thus,
    \begin{equation}
        \Lc(E^*) \leq \liminf_{n\rightarrow\infty}\Lc(E_n) = \inf_{E}\Lc(E)
    \end{equation}
    showing existence of a solution. Uniqueness is a consequence of strict convexity of the objective.
\end{proof}

Having established well-posedness we can consider the concise formulation and proof of \cref{prop:characterize_solutions}.
\begin{proposition}[Convergence for small and large regularization parameters]
Under~\cref{ass}, let $E^*(x,\lambda)$ be the minimizer of $\Lc$ for $\lambda_{\text{reg}} = \lambda$. It holds true that 
\begin{enumerate}
\item[(i)] There exists a weak accumulation point $E^*_\infty\in \W^{1,2}(\Omega)$ of $E^*(\cdot,\lambda)$ as $\lambda\rightarrow \infty$ and every such accumulation point satisfies $E^*_\infty(x) = 0$ for $p_0$-a.e. $x$.
\item[(ii)] There exists $E^*_0\in \W^{1,2}(\Omega)$ such that $\nabla E^*_0$ is a weak accumulation point of $\nabla E^*(\cdot,\lambda)$ as $\lambda\rightarrow 0$ and every such $E^*_0$ minimizes
    \begin{equation}
        \W^{1,2}(\Omega)\ni G\mapsto \iiint \bigl\| -\nabla G(\mathbf x_t) - \mathbf{u} \bigr\|^2 p(\mathbf x_0,\mathbf x_t,t)\dd \mathbf x_t\dd \mathbf x_0\dd t,
    \end{equation}
    that is, $\nabla E^*_0$ is precisely the projection of $\mathbf{u}$ onto the space of all gradient fields. Moreover, if there exists $F\in\W^{1,2}(\Omega)$ with 
    \begin{equation}
        \nabla F (x) = -\mathbb{E}[\mathbf{u}|\mathbf x_t=\mathbf x],\quad p_x-\text{a.e. $\mathbf x$}.
    \end{equation}
    then $\nabla E^*(\mathbf x) = -\mathbb{E}[\mathbf{u}|\mathbf x_t=\mathbf x]$, $p_x$-a.e.
\item[(iii)] If, in addition it holds true that $v(\mathbf x_0,\mathbf x_1,t) = v(\mathbf x_t)$, that is, $v$ depends solely on $\mathbf x_t$, then there exists a weak accumulation point $E^*_0\in \W^{1,2}(\Omega)$ of $E^*(\cdot,\lambda)$ as $\lambda\rightarrow 0$ and every such accumulation point satisfies $\nabla E^*_0 = -\mathbf{u}$, $p_x$-a.e.
\end{enumerate}
\end{proposition}
\begin{proof}
    We begin with the limit for $\lambda\rightarrow\infty$. Assume without loss of generality $\lambda\geq 1$. By optimality of $E^*(\cdot,\lambda)$
    \begin{equation}\label{eq:limit_lambda1}
        \begin{aligned}
            \int &E^*(\mathbf{x}_0,\lambda)^2 p_0(x_0) \dd \mathbf x_0\\
            &\leq\iiint \bigl\| -\nabla E^*(\mathbf{x}_t,\lambda) - \mathbf{u} \bigr\|^2  p(\mathbf x_0,\mathbf x_t,t)\dd \mathbf x_t\dd \mathbf x_0\dd t + \lambda \int E^*(\mathbf{x}_0,\lambda)^2 p_0(\mathbf x_0) \dd \mathbf x_0\\
            &\leq \iiint \bigl\| 0 - \mathbf{u} \bigr\|^2 p(\mathbf x_0,\mathbf x_t,t)\dd \mathbf x_t\dd \mathbf x_0\dd t + \lambda \int 0 \dd \mathbf x_0\\
            &\leq \iiint \bigl\|\mathbf{u} \bigr\|^2 p(\mathbf x_0,\mathbf x_t,t)\dd \mathbf x_t\dd \mathbf x_0\dd t<\infty
        \end{aligned}
    \end{equation}
    which implies that $\|E^*(\cdot,\lambda)\|_{L^2(p_0)}$ and $\|\nabla E^*(\cdot,\lambda)\|_{L^2(p_x)}$ are bounded for all $\lambda\geq 1$. As above, we can therefore extract a weakly convergent subsequence $E^*(\cdot, \lambda_n)\rightarrow E_\infty^*$ in $\W^{1,2}(\Omega)$ as well as in $L^2(p_0)$. By weak lower semi-continuity of $E\mapsto \int E(\mathbf{x})^2 p_0(\mathbf x) \dd \mathbf x$ and~\eqref{eq:limit_lambda1} it follows
    \begin{equation}
        \begin{aligned}
            \int E^*_\infty(\mathbf{x}_0)^2 p_0(\mathbf x_0) \dd \mathbf x_0
            &\leq \liminf_{n\rightarrow\infty} \int E^*(\mathbf{x}_0,\lambda_n)^2 p_0(\mathbf x_0) \dd \mathbf x_0\\
            &\leq \frac{1}{\lambda_n}\iiint \bigl\|\mathbf{u} \bigr\|^2 p(\mathbf x_0,\mathbf x_t,t)\dd \mathbf x_t\dd \mathbf x_0\dd t \rightarrow 0
        \end{aligned}
    \end{equation}
    as $n\rightarrow \infty$ implying $E^*_\infty=0$, $p(\mathbf x_0)$-a.e.


    Regarding the convergence as $\lambda\rightarrow 0$, as above we find that 
    \begin{equation}\label{eq:limit_lambda}
        \begin{aligned}
            \iiint &\bigl\| -\nabla E^*(\mathbf{x}_t,\lambda) - \mathbf{u} \bigr\|^2 p(\mathbf x_0,\mathbf x_t,t)\dd \mathbf x_t\dd \mathbf x_0\dd t\\
            &\leq \iiint \bigl\|\mathbf{u} \bigr\|^2 p(\mathbf x_0,\mathbf x_t,t)\dd \mathbf x_t\dd \mathbf x_0\dd t<\infty.
        \end{aligned}
    \end{equation}
    Thus, we can only deduce boundedness of $\nabla E^*(\cdot,\lambda)$ in $L^2(p_x)$ as $\lambda\rightarrow 0$. By the Poincar\'e inequality (\textit{cf.}~\cref{ass}), however, we also have that $F^*(\cdot,\lambda) = E^*(\cdot,\lambda) - \bar E^*(\lambda)$ is bounded in $L^2(p_x)$. Therefore, we may extract a subsequence such that $F^*(\cdot,\lambda_n)\rightharpoonup F^*_0$ as $n\rightarrow \infty$. In particular, it follows that $\nabla E^*(\cdot,\lambda_n) = \nabla F^*(\cdot,\lambda_n)\rightharpoonup \nabla F^*_0$. Once again, by weak lower semicontinuity and optimality we find for any $G\in\W^{1,2}(\Omega)$
    \begin{equation}
        \begin{aligned}
            \iiint &\bigl\| -\nabla F^*_0(\mathbf{x}_t) - \mathbf{u} \bigr\|^2 p(\mathbf x_0,\mathbf x_t,t)\dd \mathbf x_t\dd \mathbf x_0\dd t\\
            \leq& \liminf_{n\rightarrow \infty}\iiint \bigl\| -\nabla E^*(\mathbf{x}_t,\lambda_n) - \mathbf{u} \bigr\|^2 p(\mathbf x_0,\mathbf x_t,t)\dd \mathbf x_t\dd \mathbf x_0\dd t\\
            \leq& \liminf_{n\rightarrow \infty}\iiint \bigl\| -\nabla E^*(\mathbf{x}_t,\lambda_n) - \mathbf{u} \bigr\|^2 p(\mathbf x_0,\mathbf x_t,t)\dd \mathbf x_t\dd \mathbf x_0\dd t + \lambda_n \int E^*(\mathbf{x}_0,\lambda)^2 p(\mathbf x_0) \dd \mathbf x_0\\
            \leq& \liminf_{n\rightarrow \infty}\iiint \bigl\| -\nabla G(\mathbf x_t) - \mathbf{u} \bigr\|^2 p(\mathbf x_0,\mathbf x_t,t)\dd \mathbf x_t\dd \mathbf x_0\dd t + \lambda_n \int G(\mathbf x_0)^2 p(\mathbf x_0) \dd \mathbf x_0\\
            \leq& \iiint \bigl\| -\nabla G(\mathbf x_t) - \mathbf{u} \bigr\|^2 p(\mathbf x_0,\mathbf x_t,t)\dd \mathbf x_t\dd \mathbf x_0\dd t
        \end{aligned}
    \end{equation}
    implying the desired result. The representation, $\nabla E^*(\mathbf x) = -\mathbb{E}[\mathbf{u}|\mathbf x_t=\mathbf x]$ follows directly by the fact that 
    \begin{equation}
        w\mapsto \iiint \bigl\| w(\mathbf x_t) - \mathbf{u} \bigr\|^2 p(\mathbf x_0,\mathbf x_t,t)\dd \mathbf x_t\dd \mathbf x_0\dd t
    \end{equation}
    is minimized for $w(\mathbf x) = \mathbb{E}[\mathbf{u}|\mathbf x_t=\mathbf x]$.
    
    In the last case, we note that 
    \begin{equation}
        \begin{aligned}
            \iiint &\bigl\| w(\mathbf x_t) - \mathbf{u} \bigr\|^2 p(\mathbf x_0,\mathbf x_t,t)\dd \mathbf x_t\dd \mathbf x_0\dd t\\
            =&\int \bigg(\iint \bigl\| w(\mathbf x_t) - \mathbf{u} \bigr\|^2 p(\mathbf x_0,t|\mathbf x_t)\dd \mathbf x_0\dd t \bigg) p(\mathbf x_t)\dd \mathbf x_t\\
        \end{aligned}
    \end{equation}
    which is zero for $w(\mathbf x_t) = \mathbf{u}(\mathbf x_t)$. Therefore, under the assumption that there exists $F\in\W^{1,2}(\Omega)$ with $\nabla F(\mathbf x_t) = \mathbf{u}(\mathbf x_t)$, we find by optimality of $E^*(\cdot,\lambda)$
    \begin{equation}\label{eq:limit_lambda}
        \begin{aligned}
            \iiint &\bigl\| -\nabla E^*(\mathbf{x}_t,\lambda) - \mathbf{u} \bigr\|^2 p(\mathbf x_0,\mathbf x_t,t)\dd \mathbf x_t\dd \mathbf x_0\dd t + \lambda_n \int E^*(\mathbf{x}_0,\lambda)^2 p_0(\mathbf x_0) \dd \mathbf x_0\\
            &\leq \iiint \bigl\| \nabla F(\mathbf x_t) - \mathbf{u} \bigr\|^2 p(\mathbf x_0,\mathbf x_t,t)\dd \mathbf x_t\dd \mathbf x_0\dd t + \lambda_n \int F(\mathbf x_0)^2 p_0(\mathbf x_0)\dd \mathbf x_0\\
            &= \lambda_n \int F(\mathbf x_0)^2 p_0(\mathbf x_0)\dd \mathbf x_0
        \end{aligned}
    \end{equation}
    so that in this case obtain again boundedness of $\nabla E^*(\cdot,\lambda)$ and $E^*(\cdot,\lambda)$ both in $L^2(p_x)$. Extracting weakly convergent subsequences again with limit $E^*_0$ we find by weak lower semi-continuity
    \begin{equation}
        \begin{aligned}
            \iiint &\bigl\| \nabla E^*_0(\mathbf x_t) - \mathbf{u} \bigr\|^2 p(\mathbf x_0,\mathbf x_t,t)\dd \mathbf x_t\dd \mathbf x_0\dd t\\
            &\leq \liminf_{n\rightarrow \infty}\iiint \bigl\| \nabla E^*(\mathbf x_t,\lambda_n) - \mathbf{u} \bigr\|^2 p(\mathbf x_0,\mathbf x_t,t)\dd \mathbf x_t\dd \mathbf x_0\dd t\\
            &\leq \liminf_{n\rightarrow \infty}\iiint \bigl\| \nabla E^*(\mathbf x_t,\lambda_n) - \mathbf{u} \bigr\|^2 p(\mathbf x_0,\mathbf x_t,t)\dd \mathbf x_t\dd \mathbf x_0\dd t+ \lambda_n \int E^*(\mathbf{x}_0,\lambda)^2 p_0(\mathbf x_0) \dd \mathbf x_0\\
            &\leq \liminf_{n\rightarrow \infty}\lambda_n \int F(\mathbf{x}_0,\lambda)^2 p_0(\mathbf x_0) \dd \mathbf x_0=0
        \end{aligned}
    \end{equation}
    concluding the proof.
\end{proof}

\section{Empirical Verification of Local Minima}
\label{sec:landscape_verification}
The RFM construction aims to place data points at local minima of the learned  energy, as the smoothing function ensures pointwise stationarity at $t=0$, and the symmetric restoring targets induce non-negative curvature around data. The theoretical analysis in Section~\ref{sec:analysis} establishes existence and uniqueness of the loss minimizer and characterizes its behavior under varying regularization strength, but does not formally prove the local-minima property for the marginalized predictor. Here we verify empirically that the trained network $E_\theta$ realizes this intended behavior through three complementary experiments.

\paragraph{Deterministic gradient relaxation from data.} We initialize 1000 molecules taken from the training and test sets and run deterministic gradient descent ($\tau=0$) for $N=500$ steps with step size $\eta=0.01$ on the learned energy landscape. We evaluate atom displacement via mean RMSD, the energy change during relaxation, and the change in molecule stability. We conduct this experiment for two models: EBMol trained with the full RFM objective, and the OTFM ablation from Table~\ref{tab:learning_ablation}, which is trained without the extrapolation construction.

\begin{table}[ht]
  \centering
  \caption{Deterministic gradient relaxation initialized at data. Energy values, RMSD (in $\AA$), and stability after $N$
 steps of gradient descent on $E_\theta$.}
  \label{tab:relaxation}
  \begin{tabular}{l l c c c c c}
    \toprule
    \textbf{Method}& \textbf{Split} &\textbf{Median $\Delta E$ $\downarrow$} & \textbf{Mean RMSD $\downarrow$} & \textbf{$\Delta$Mol. Stability} $\downarrow$\\
    \midrule
    EBMol (RFM)  & Train&  -0.59 & 0.121 & 2.4 \\
    EBMol (RFM)  & Test &  -0.67 & 0.124 & 2.1 \\
    \midrule
    OTFM (\ref{tab:learning_ablation}) & Train & -154.98 & 143.365 & 63.1\\
    \bottomrule
  \end{tabular}
\end{table}

Results are reported in Table~\ref{tab:relaxation}. Three observations confirm the role of the RFM construction. First, on training data, EBMol configurations remain effectively stationary (RMSD $0.121\,\AA$, stability loss $2.4$ percentage points), verifying that the RFM loss produces energy minima at training samples as intended. Second, the same holds on held-out test data (RMSD $0.124\,\AA$, stability loss $2.1$ percentage points), demonstrating that the local-minima property generalizes to molecules unseen during training. Third, under the OTFM loss, configurations diverge dramatically (RMSD $143.4\,\AA$) into spurious minima that do not correspond to valid molecular structures, with molecule stability dropping by $63$ percentage points. This confirms that the extrapolation construction of RFM is responsible for placing data at local minima.

\paragraph{Energy and gradient profiles along interpolation paths.} 
We evaluate the gradient magnitude with respect to the position $\|\nabla_c E_\theta(\mathbf{x}_t)\|$ and the cosine similarity between $\nabla_c E_\theta(\mathbf{x}_t)$ and the unsmoothed restoring direction ${\mathbf{u}}^{\mathbf{c}} = \mathrm{sign}(t) \cdot (\mathbf{c}_0 - \mathbf{c}_1),$ along linear interpolation paths between training molecules $(t=0)$ and corresponding optimal-transport-matched prior samples $(t=\pm 1)$. 

Figure~\ref{fig:grad_analysis} shows the results averaged over 1000 molecules. The gradient magnitude drops sharply near $t=0$, consistent with data sitting at stationary points of the learned energy. Away from data, the gradient grows monotonically toward the prior, indicating increasingly steep energy walls around each basin. The cosine similarity remains high ($>0.6$) throughout most of the interval, confirming that the learned gradient field points back toward data across the interpolation range, and drops near $t=0$ where the vanishing gradient magnitude makes the direction ill-defined. The symmetric profile across $t<0$ and $t>0$ reflects the symmetric construction of the RFM target.
\begin{figure}
    \centering
    \includegraphics[width=0.75\linewidth]{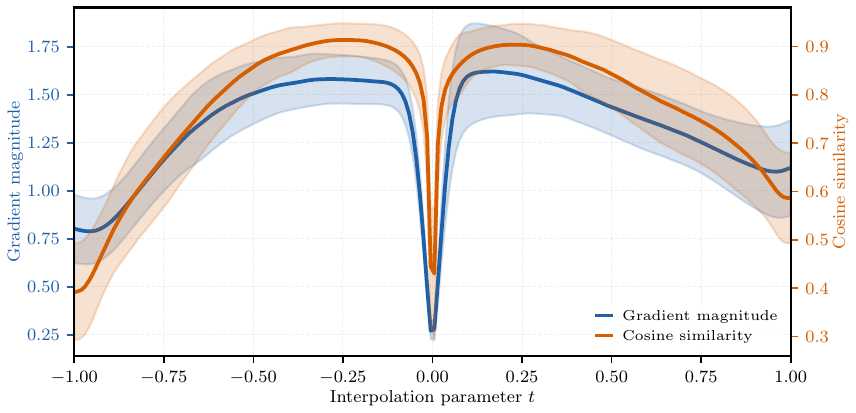}
    \caption{Gradient (wrt. position) magnitude (blue) and cosine similarity between the learned restoring field and the prescribed target direction (orange) along interpolation paths from prior ($t=\pm 1$) through data ($t=0$). Shaded regions indicate one standard deviation over 1000 molecules. The gradient vanishes near data, confirming stationarity, while the high cosine similarity away from data confirms that the learned field consistently points toward training molecules.}
    \label{fig:grad_analysis}
\end{figure}
    
\paragraph{Response to Gaussian perturbation.}
We perturb training molecules by adding isotropic Gaussian noise $\boldsymbol{\epsilon} \sim \mathcal{N}(0, \sigma^2 I)$ to atomic coordinates for increasing noise levels $\sigma$ and evaluate $E_\theta$ at the perturbed configurations. Figure~\ref{fig:energy_noise} shows the learned energy as a function of $\sigma$, averaged over 1000 molecules. The energy increases smoothly and monotonically with perturbation magnitude, starting near zero at clean data. This confirms that data points sit at the center of well-shaped energy basins, with the learned landscape assigning progressively higher energy to configurations that deviate from valid molecular geometries.
\begin{figure}
    \centering
    \includegraphics[width=0.75\linewidth]{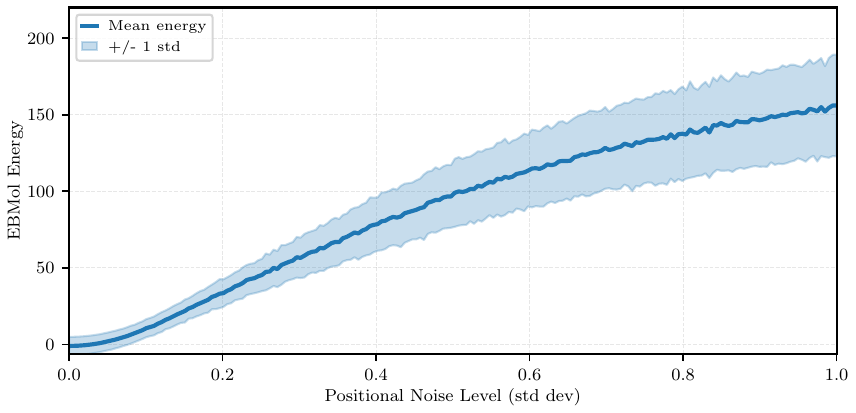}
    \caption{Learned energy $E_\theta$ as a function of Gaussian noise magnitude $\sigma$ applied to training molecule coordinates. Shaded regions indicate one standard deviation over 1000 molecules. The monotonic increase confirms that data points occupy local minima of the learned landscape.}
    \label{fig:energy_noise}
\end{figure}

\section{Structural and energetic analysis}
\label{sec:energy_correlation_appendix}

\paragraph{Bond distance distributions.} Figure~\ref{fig:bond_distances} provides a per-element decomposition of the aggregate bond length errors reported in Table~\ref{tab:geom_drugs_energy}, comparing bond distance distributions for molecules generated by SLDM and EBMol$_{7240}$ against the training data on GEOM-Drugs. EBMol recovers the C--C distribution with high fidelity, accurately capturing peak locations, relative heights, and the separation between single, aromatic, and double bond modes. For C--N and C--O, EBMol correctly resolves relative peak proportions, which SLDM fails to achieve. EBMol also successfully recovers the bimodality of C--H bonds.

\begin{figure}
\centering
\includegraphics[width=\linewidth]{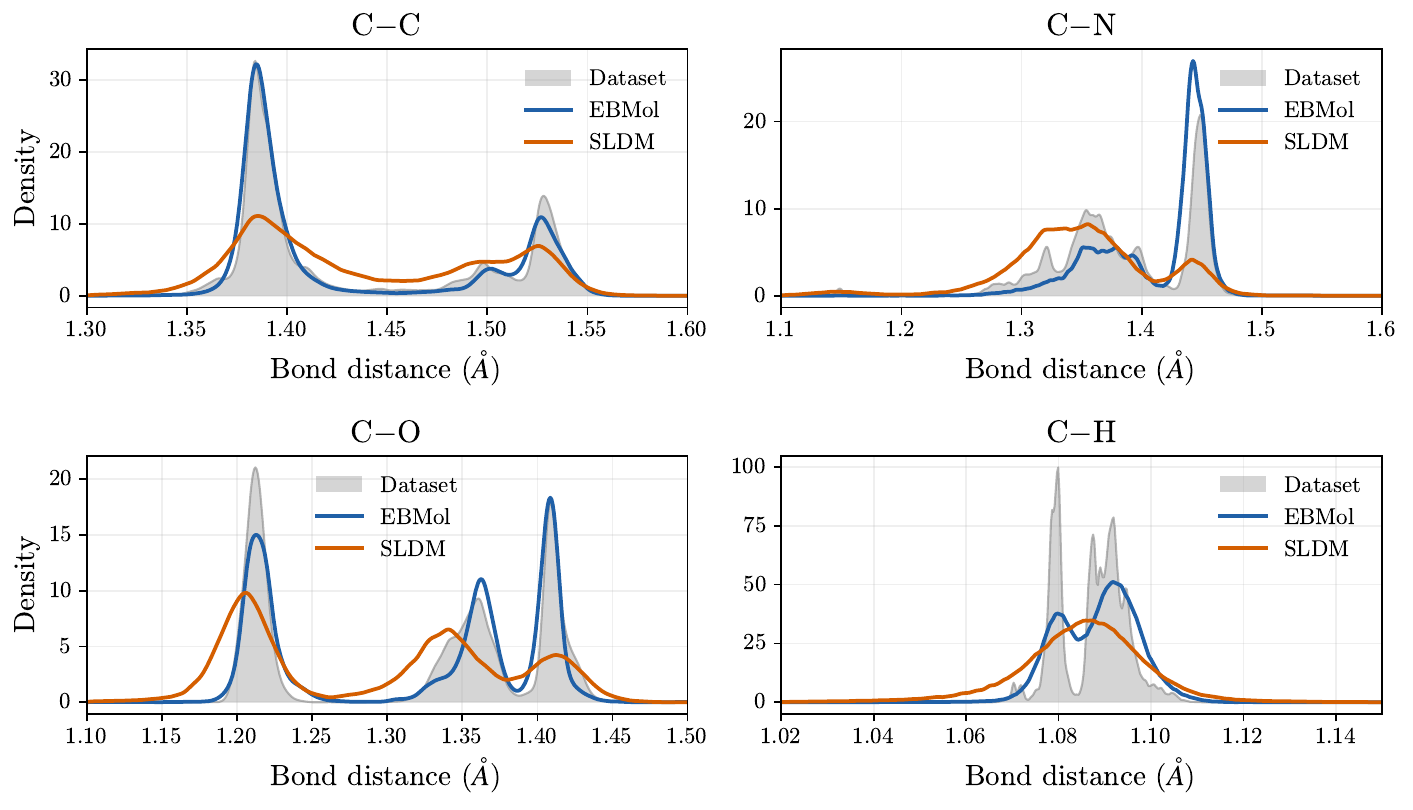}
\caption{Bond distance distributions for C--C, C--N, C--O, and C--H pairs on GEOM-Drugs, comparing EBMol$_{7240}$ (blue) and SLDM (orange) against the dataset (gray). Distinct peaks correspond to different bond orders. EBMol recovers the multimodal structure of the reference distributions, with sharper peak separation and more realistic proportions than SLDM.}
\label{fig:bond_distances}
\end{figure}

\paragraph{Cross-model energy correlation.}Figure~\ref{fig: energy correlation} plots EBMol's learned energy against the energy decrease upon GFN2-xTB relaxation for 10,000 molecules generated by EDM. The positive correlation and its interpretation are discussed in Section~\ref{sec: physically meaningful}.
\begin{figure}
    \centering
    \includegraphics[width=1\linewidth, trim={0 0.35cm 0 0.15cm}, clip]{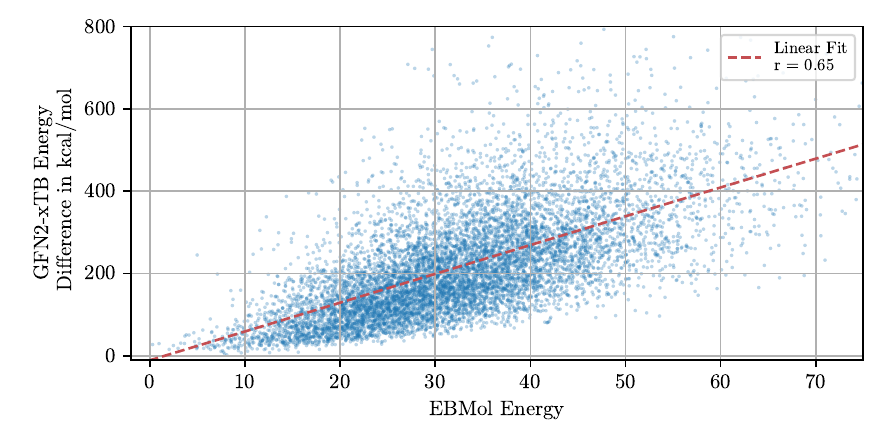}
    \caption{The learned energy correlates with physical quality across models. EBMol energy vs.\ energy decrease upon GFN2-xTB relaxation for 10k molecules generated by EDM ($r = 0.65$, $p < 0.0001$). Higher EBMol energy predicts larger relaxation energy, indicating physically strained geometries.}
    \label{fig: energy correlation}

\end{figure}

\section{Perception pipeline analysis and failure modes on GEOM-Drugs}
\label{app:failure_modes}

Bonds, bond orders, formal charges, and aromaticity are not generated by the model but inferred post-hoc from coordinates and atom types via OpenBabel~\cite{oboyle_open_2011}. A failed stability or validity check can therefore either be caused by an erroneous generated structure or from a limitation of the perception pipeline itself. To separate the two, we first apply the pipeline to the GEOM-Drugs reference conformers, for which the ground-truth SMILES allows each failure to be attributed to its cause, and subsequently report the same diagnostics for generated molecules.

\paragraph{Failure modes of the evaluation pipeline.}
Although every reference conformer corresponds to a valid molecule, only 88.2\% pass the molecule stability check (Table~\ref{tab:geom_results}). In Table~\ref{tab:failure_reference}, we attribute the remaining 11.8\% to its cause. The dominant failure is over-valence, which co-occurs predominantly with formal charges in the reference SMILES. OpenBabel does not reliably recover charged forms from geometry and the uncharged perception then violates the valency tables. Under-valence arises from bonds and aromatic systems that OpenBabel fails to perceive. Notably, OpenBabel issues a kekulization warning for only 1.3\% of reference molecules, whereas 3.5\% contain an aromatic system that is not perceived, indicating that the warning is not a reliable indicator of kekulization failures.

\begin{table}[h]
\centering
\caption{Failure modes of the evaluation pipeline on GEOM-Drugs reference conformers. A molecule is over-/under-valent if at least one atom's perceived bond-order sum lies above/below its allowed valence. Multiple failure modes may occur per molecule, so columns do not sum to 100\%.}
\label{tab:failure_reference}
\begin{tabular}{lcc}
\toprule
Failure mode & \% of all & \% of unstable \\
\midrule
Disconnected                          & 0.00 & 0.0 \\
Over-valent                           & 7.70 & 65.4 \\
Under-valent                          & 6.00 & 50.9 \\
Formal charge in reference SMILES     & 5.96 & 50.4 \\
Aromatic system not perceived         & 3.54 & 30.1 \\
RDKit sanitization failure            & 2.10 & 17.7 \\
OpenBabel kekulization warning        & 1.30 & 11.0 \\
Connectivity differs from reference   & 0.32 & 2.7 \\
\bottomrule
\end{tabular}
\end{table}

\paragraph{Failure modes of generated molecules.}
For generated molecules no reference SMILES is available, so only the pipeline's own signals can be reported. Results are shown in Table~\ref{tab:failure_generated}. EBMol's failures are dominated by under-valence, i.e.\ atoms whose perceived bond orders fall short of their allowed valence, while disconnected and over-valent molecules are comparatively rare. All valence-related failure rates, as well as disconnectedness and sanitization failures, decrease monotonically with the sampling budget. In contrast, EDM and SLDM fail primarily through disconnected fragments (57.6\% and 49.4\%).

The kekulization warning should be interpreted with care. OpenBabel must perceive an aromatic system before kekulization can fail on it, so the warning rate is biased downward for models that produce few perceivable aromatic systems. EDM, the least stable model, yields only 4.5\% aromatic atoms against 26.2\% in the data. We therefore report the kekulization warning as a diagnostic rather than as an aromaticity error rate.

\begin{table}[h]
\centering
\small
\setlength{\tabcolsep}{4pt}
\caption{Failure modes of generated molecules on GEOM-Drugs (\% of all molecules). Over-/under-valent as in Table~\ref{tab:failure_reference}. Rows overlap and report per-molecule presence.}
\label{tab:failure_generated}
\begin{tabular}{lccccccc}
\toprule
 & Data & EBMol$_{1080}$ & EBMol$_{1960}$ & EBMol$_{3720}$ & EBMol$_{7240}$ & EDM & SLDM \\
\midrule
Valid $\uparrow$                     & 97.9 & 95.4 & 97.0 & 98.6 & 99.0 & 84.4 & 87.0 \\
Stable $\uparrow$                    & 88.2 & 24.7 & 37.4 & 52.3 & 61.4 & 1.0  & 20.6 \\
\midrule
Disconnected $\downarrow$            & 0.0  & 14.4 & 9.1  & 4.7  & 3.8  & 57.6 & 49.4 \\
Sanitization failure $\downarrow$    & 2.1  & 4.6  & 3.0  & 1.4  & 1.0  & 15.6 & 13.0 \\
Over-valent $\downarrow$             & 7.7  & 8.8  & 6.5  & 4.6  & 4.1  & 18.4 & 25.4 \\
Under-valent $\downarrow$            & 6.0  & 73.2 & 60.3 & 45.5 & 36.4 & 98.7 & 74.6 \\
Aromatic atoms perceived             & 26.2 & 17.5 & 19.1 & 21.0 & 22.0 & 4.5  & 17.5 \\
Kekulization warning                 & 1.3  & 21.9 & 20.8 & 18.9 & 17.9 & 23.6 & 14.9 \\
\bottomrule
\end{tabular}
\end{table}

\section{Sampling Cost}
\label{app:sampling_cost}

\paragraph{Wall-clock sampling cost.}
Table~\ref{tab:sampling_throughput} reports wall-clock sampling cost on a single NVIDIA RTX 4090, measured both per sample and per atom. We additionally report per-atom cost because the parallel-tempering (PT) sampler spends less compute on smaller molecules, which per-sample numbers alone would hide. Regarding implementation, we use a PyTorch-Geometric-based implementation that batches molecules as disconnected graphs, where EDM and SLDM pad every molecule to the GEOM maximum of 181 atoms instead. Although EBMol runs at substantially higher NFE, its throughput is comparable to the diffusion and flow baselines.

\begin{table}[h]
\centering
\caption{Sampling throughput on a single NVIDIA RTX 4090. The suffix denotes the number of function evaluations (NFE).}
\label{tab:sampling_throughput}
\begin{tabular}{lcc}
\toprule
Method & atoms/s & s/sample \\
\midrule
EBMol$_{1080}$ & 38.4 & 1.07 \\
EBMol$_{1960}$ & 18.5 & 2.19 \\
EBMol$_{3720}$ & 9.4  & 4.24 \\
EBMol$_{7240}$ & 5.2  & 8.04 \\
\midrule
SLDM$_{200}$   & 10.4 & 4.49 \\
SLDM$_{1000}$  & 4.3  & 10.91 \\
EDM$_{1000}$   & 5.8  & 8.03 \\
\bottomrule
\end{tabular}
\end{table}

\paragraph{Cost per high-quality sample.}
To make that tradeoff between sampling cost and sample quality more explicit, we report the cost per high-quality (hq) sample. We define a sample as high-quality if it is stable, valid, connected, and has a GFN2-xTB relaxation energy below a threshold $\tau$. We measure cost in NFEs to decouple it from implementation-specific overheads such as padding (Table~\ref{tab:nfe_per_hq}).

\begin{table}[h]
\centering
\small
\setlength{\tabcolsep}{4pt}
\caption{Share of high-quality samples and NFEs per high-quality sample for different relaxation-energy thresholds $\tau$ (kcal/mol). Bold marks the lowest cost per threshold.}
\label{tab:nfe_per_hq}
\begin{tabular}{l cc cc cc cc}
\toprule
& \multicolumn{2}{c}{$\tau = 15$} & \multicolumn{2}{c}{$\tau = 5$} & \multicolumn{2}{c}{$\tau = 2$} & \multicolumn{2}{c}{$\tau = 1$} \\
\cmidrule(lr){2-3}\cmidrule(lr){4-5}\cmidrule(lr){6-7}\cmidrule(lr){8-9}
Method & hq (\%) & NFE/hq & hq (\%) & NFE/hq & hq (\%) & NFE/hq & hq (\%) & NFE/hq \\
\midrule
EBMol$_{1080}$ & 22.7 & \textbf{4766}  & 18.1 & \textbf{5964}  & 11.5 & \textbf{9412} & 4.8  & 22548 \\
EBMol$_{1960}$ & 35.1 & 5590           & 30.6 & 6412           & 20.8 & 9433          & 9.4  & 20769 \\
EBMol$_{3720}$ & 50.2 & 7413           & 46.6 & 7990           & 35.4 & 10502         & 18.1 & \textbf{20550} \\
EBMol$_{7240}$ & 59.2 & 12240          & 56.5 & 12806          & 45.5 & 15928         & 24.2 & 29906 \\
\midrule
SLDM$_{1000}$  & 7.6  & 13081          & 1.9  & 53951          & 0.2  & 665400        & 0.04 & 2495250 \\
\bottomrule
\end{tabular}
\end{table}

EBMol requires substantially less compute per high-quality sample than SLDM at every threshold. The gap is largely due to SLDM’s low hq share, which is mostly caused by disconnected samples. Stricter energy thresholds raise the cost per hq sample for every configuration.They also shift the most compute-efficient EBMol configuration toward higher sampling budgets: EBMol-1080 is the most economical for $\tau \geq 2$\,kcal/mol, while EBMol-3720 is the cheapest at $\tau = 1$\,kcal/mol.

\section{First-pass filtering experiments}
\label{app:first_pass_filter}

We evaluate whether the learned energy can serve as a fast first-pass filter for identifying low-relaxation-energy molecules among sampled generated by other generative models. Ground-truth labels are GFN2-xTB relaxation energies, and MMFF94 serves as the fast-filter baseline. We filter the outputs of the generative models EDM~\cite{hoogeboomEquivariantDiffusionMolecule2022}, GeoLDM~\cite{xuGeometricLatentDiffusion2023}, and SLDM~\cite{niStraightLineDiffusionModel2025a}.

\paragraph{Throughput.}
Table~\ref{tab:filter_throughput} reports throughput on 9112 molecules on a single NVIDIA RTX 4090. Energy evaluation with EBMol is a single batched forward pass, so its throughput matches MMFF94 and is roughly $500\times$ higher than GFN2-xTB relaxation.

\begin{table}[h]
\centering
\caption{Throughput for first-pass filtering on 9112 molecules.}
\label{tab:filter_throughput}
\begin{tabular}{lcl}
\toprule
Method & Throughput (mol/s) & Notes \\
\midrule
EBMol    & 2850 & single forward pass, batch size 500 (GPU) \\
MMFF94   & 2204 & \\
GFN2-xTB & 6.1  & 16 CPU threads \\
\bottomrule
\end{tabular}
\end{table}

\paragraph{Retrieval performance.}
We define the low-energy set as the lowest decile of molecules by GFN2-xTB relaxation energy and measure how well each filter's ranking recovers it. Precision@$k$ is the fraction of the top-$k$ ranked molecules that belong to the low-energy set; recall@$k$ is the fraction of the low-energy set contained in the top-$k$. A random ranking achieves precision $0.10$ at every $k$ and recall@25\% of $0.25$.

\begin{table}[h]
\centering
\caption{First-pass filtering performance (EBMol / MMFF94). The low-energy set is the lowest decile by GFN2-xTB relaxation energy.}
\label{tab:filter_retrieval}
\begin{tabular}{lccc}
\toprule
Metric & EDM & GeoLDM & SLDM \\
\midrule
Precision@1\%  & \textbf{0.90} / 0.47 & \textbf{0.93} / 0.32 & \textbf{0.42} / 0.12 \\
Precision@10\% & \textbf{0.54} / 0.41 & \textbf{0.54} / 0.29 & \textbf{0.37} / 0.15 \\
Recall@25\%    & \textbf{0.84} / 0.75 & \textbf{0.86} / 0.65 & \textbf{0.71} / 0.40 \\
\bottomrule
\end{tabular}
\end{table}

We identify and test two use-cases and report results in Table~\ref{tab:filter_retrieval}. When expensive relaxation is applied only to the top-ranked candidates, precision@1\% is the relevant metric as it ensures that compute is only spent on the most promising candidates. For EDM and GeoLDM samples, 90\% and 93\% of EBMol's top 1\% belong to the low-energy set, so nearly the entire relaxation budget is spent on true hits. When the goal is to discard poor samples while retaining good ones, recall@25\% is the relevant metric: keeping only the top quarter of the pool preserves 84--86\% of the low-energy set on EDM and GeoLDM and 71\% for SLDM. Across all three generators and all operating points, EBMol outperforms MMFF94 at comparable throughput.

\section{Training objective and sampler ablation}
\label{sec: model sampler ablation}
\paragraph{Training objective ablation.} A learned energy landscape requires a sampler to generate samples.  
We therefore evaluate each training objective under three different samplers (Forward Euler (FWDE), Annealed Langevin Dynamics (ALD), and Parallel Tempering (PT)) to disentangle the contributions of training and sampling. All ablations use a reduced-size model with embedding dimension halved to 128 to keep the search tractable. To ensure a fair comparison, we tune sampler hyperparameters separately for each training objective. For FWDE, we sweep the step size over $[0.05, 0.001]$ and select the configuration with the highest atom stability. For the FM baselines this controls how finely the $[0,1]$ integration interval is discretized, giving them the best possible performance under their natural sampling procedure. For ALD, we lock in the selected step size and additionally sweep over the noise scale $\sigma$, using the same value for coordinates and atom types for simplicity. For PT, we reuse the ALD configuration. We acknowledge this sequential scheme is suboptimal but it keeps the search feasible across all method-sampler combinations. All experiments use a fixed budget of 1080 NFE per sample on average and $n=500$ samples per experiment. In addition to atom stability and validity, we report two metrics specific to this ablation: W2, the Wasserstein-2 distance between pairwise atom distance distributions of generated and reference molecules restricted to the $1.0$--$2.0$\,\AA\ range most relevant for bond lengths, and the number of diverged samples that exceeded plausible energy bounds during sampling and were discarded. Results are reported in Table~\ref{tab:learning_ablation}.

We ablate against two baselines that share EBMol's energy architecture but are trained with the flow matching loss: Standard FM (without optimal transport or atom permutation alignment) and Equivariant OT-FM (with SE(3)-equivariant optimal transport coupling and atom permutation). The gap between the two is substantial, confirming that equivariant OT alignment is critical for learning a useful energy landscape in our framework. However, Equivariant OT-FM produces a significant number of diverged samples across all samplers, consistent with the relaxation experiments in Section~\ref{sec:landscape_verification}: without the RFM construction, data points are not local minima, and gradient-based sampling can escape into unphysical regions. This problem is most severe under PT, where high-temperature chains actively encourage exploration of the landscape. We note that the 1074 diverged samples for only 500 generated samples is possible because our parallel tempering implementation replaces diverged samples with fresh prior samples periodically. This also precludes inference-time compute scaling, since increasing the sampling budget leads to more divergences rather than better samples.

The RFM block tells a different story. No RFM variant produces a single diverged sample, confirming that the restoring-field construction yields a landscape that is globally stable under sampling. Within the RFM variants, several trends emerge. Comparing RFM with and without the smoothing function $b(t)$ shows that smoothing improves the W2 bond-distance metric, suggesting that the smoother basins are easier for the sampler to navigate, though PT partially compensates for the absence of smoothing. Moving from a static isotropic prior to the data-driven anisotropic prior described in the Implementation section yields a further improvement, most visible in the Valid \& Connected rate: the full RFM model achieves 37.7\% under PT, compared to 22.6\% for the static-prior variant. This indicates that the data-driven prior helps the sampler to find not just locally valid atomic environments (atom stability) but globally coherent molecular structures (validity and connectivity).

\begin{table}[t]
\centering
\caption{Ablation of training objectives on GEOM-Drugs using a reduced-size model. ALD: Annealed Langevin Dynamics; PT: Parallel Tempering; FWDE: Forward Euler. No $b(t)$ omits smoothing and static prior uses a single fixed prior during training. Best overall per metric in \textbf{bold}, best within each block \underline{underlined}.}
\label{tab:learning_ablation}
\resizebox{\textwidth}{!}{
\begin{tabular}{l c c c c c}
\toprule
\textbf{Method} & \textbf{Sampler} & \textbf{Atom} & \textbf{Valid \&} & \textbf{W2 Dist.} & \textbf{Diverged} \\
&  &\textbf{Stability $\uparrow$} &  \textbf{Connected $\uparrow$} &\textbf{($\times 10^{-2}$) $\downarrow$} & \textbf{Samples} $\downarrow$\\
\midrule
\multicolumn{6}{l}{\textit{FM-based baselines}} \\
\midrule
Standard FM                   & FWDE & 40.4 & 0.00 & 5.3 & 0 \\
Standard FM                   & ALD & 40.3 & 0.00 & 5.2 & 0 \\
Standard FM                   & PT & 40.1 & 0.00 & 5.3 & 0 \\

Equivariant OT-FM             & FWDE & 60.5 & {1.6} & 3.3 & 40 \\
Equivariant OT-FM             & ALD & \underline{64.6} & 2.2 & \underline{0.6} & 51 \\
Equivariant OT-FM             & PT & 53.7 & \underline{14.9} & 0.8 & 1074 \\
\midrule
\multicolumn{6}{l}{\textit{RFM ablations}} \\
\midrule
RFM, no $b(t)$, static prior  & FWDE & 76.1 & 9.4 & 0.8 & 0 \\
RFM, no $b(t)$, static prior  & ALD & 79.7 & 16.8 & 0.6 & 0 \\
RFM, no $b(t)$, static prior  & PT & \textbf{85.2} & 22.4 & \underline{0.4} & 0 \\

RFM, static prior             & FWDE & 78.6 & 15.8 & \underline{0.4} & 0 \\
RFM, static prior             & ALD & 79.1 & 13.0 & \underline{0.4} & 0 \\
RFM, static prior             & PT & 81.6 & \underline{22.6} & 0.5 & 0 \\

\midrule
RFM                         & FWDE & {79.6} & {23.6} & {0.4} & 0 \\
RFM                         & ALD & {81.6} & {33.0} & {0.3} & 0 \\
RFM                         & PT & \underline{83.6} & \textbf{37.7} & \textbf{0.2} & 0 \\

\bottomrule
\end{tabular}
}
\end{table}

We ablate sampler choices for our full RFM model used throughout this paper, evaluating on GEOM-Drugs with $n=200$ samples. We follow the same incremental hyperparameter search as before: sweep step size for FWDE, lock it in and additionally sweep $\sigma$ for ALD. For PT we use the hyperparameters from Table~\ref{tab:hyperparameters}. 

Results are reported in Table~\ref{tab:sampling_ablation}. The results show that noise is critical for sampling quality, as FWDE ($\sigma=0$) achieves reasonable atom stability but near-zero molecule stability and the highest median energy, indicating that deterministic trajectories get trapped in shallow local minima. ALD improves substantially across all metrics, confirming that stochastic exploration helps escape these shallow basins. However, ALD tends to produce locally valid fragments rather than globally coherent molecules, as evidenced by the gap between its atom stability and Valid \& Connected rate. Both PT variants close this gap, indicating that the temperature ladder successfully enables basin-to-basin transitions. A surprising observation is the significant difference between the two PT variants, as Atom Max nearly doubles molecule stability compared to Atom Avg at matched NFE (24.5\% vs. 15.5\% at 1080 steps). We additionally test inference-time compute scaling by increasing the NFE budget from 1080 to 1960 steps. All samplers improve, with PT + Atom Max showing the largest gains, reaching 40.0\% molecule stability.

\begin{table}[ht]
    \centering
    \caption{Ablation of sampling algorithms for our RFM-trained model on GEOM-Drugs ($n=200$ samples). ALD: Annealed Langevin Dynamics; PT: Parallel Tempering;  FWDE: Forward Euler.}
    \label{tab:sampling_ablation}
    \begin{tabular}{@{}lcccccc@{}}
        \toprule
        \textbf{Algorithm} & \textbf{(Avg.)} & \textbf{Atom } & \textbf{Molecule} & \textbf{Valid \&} & \textbf{Diversity} & \textbf{Median}\\
         & \textbf{Steps}& \textbf{Stab. ($\uparrow$)} &  \textbf{Stab. ($\uparrow$)} & \textbf{Connected($\uparrow$)} & \textbf{($\uparrow$)} & \textbf{Energy} ($\downarrow$)\\
        \midrule
        FWDE        & 1080 & 88.2 & 0.5  & 48.0 & 129.8 & 21.7 \\
        FWDE        & 1960 & 88.1 & 0.5  & 50.0 & 130.0 & 20.8 \\
        ALD            & 1080 & 93.0 & 18.0 & 61.0 & 110.0 & 7.7 \\
        ALD            & 1960 & 95.5 & 23.5 & 70.5 & 108.9 & 4.6 \\
        PT + Atom Avg. & 1080 & 92.9 & 15.5 & 78.0 & 116.4 & 3.8 \\
        PT + Atom Avg. & 1960 & 95.5 & 22.0 & 82.5 & 113.0 & 1.6\\
        PT + Atom Max  & 1080 & 95.6 & 24.5 & 83.0 & 114.5 & 2.5 \\
        PT + Atom Max  & 1960 & 97.3 & 40.0 & 89.5 & 108.1 & 1.3 \\

        \bottomrule
    \end{tabular}
\end{table}

\section{Hyperparameters and experiment details}
\label{sec: hyperparameters and details}
\paragraph{Implementation.} We implement EBMol in PyTorch using PyTorch Geometric~\cite{feyPyG20Scalable2025} for their efficient batched graph operations. SE(3) and permutation alignment uses the efficient CPU-based optimal-transport solver of~\cite{kleinEquivariantFlowMatching2023, songEquivariantFlowMatching2023}. We use exponential moving average (EMA) for evaluation and sampling. Dataset splits for QM9 and GEOM-Drugs follow~\cite{hoogeboomEquivariantDiffusionMolecule2022}.

\paragraph{Compute.} The QM9 model was trained for approximately 4 days on a single NVIDIA RTX 3090 Ti. The GEOM-Drugs model was trained for approximately 3 days on 2 $\times$ NVIDIA L40.

\begin{table}[ht]
\centering
\caption{Full hyperparameters.}
\small
\label{tab:hyperparameters}
  \begin{tabular}{l c c}
    \toprule
    \textbf{Parameter} & \textbf{QM9} & \textbf{GEOM-Drugs} \\
    \midrule
    \multicolumn{3}{c}{\textbf{Network}} \\
    \midrule
    EGNN layers  & 9 & 4 \\
    Embedding dimension & 256 & 128 \\
    Energy MLP layers & 2 & 2 \\
    Energy MLP hidden dim  & 256 & 128\\
    Total parameters & 5.2M & 2.2M \\
    \midrule
    \multicolumn{3}{c}{\textbf{Training}} \\
    \midrule
    \multicolumn{3}{l}{\textit{Optimization}} \\
    \midrule
    Optimizer & Adam & Adam \\
    Learning rate &5e-5 & 5e-5  \\
    LR schedule & Linear & Linear\\
    Batch size & 64 & 16 \\
    Training epochs & 750 & 7 \\
    EMA decay & 0.999 & 0.999\\
    Gradient clipping & Yes & Yes\\
    \midrule
    \multicolumn{3}{l}{\textit{RFM loss}} \\
    \midrule
    Smoothing function $b(s)$ & $\tanh(\gamma s)$ & $\tanh(\gamma s)$\\
    Smoothing sharpness $\gamma$ & 25.0 & 25.0 \\
    Energy regularization $\lambda_{\text{reg}}$ & 1e-3 & 1e-3 \\
    \midrule
    \multicolumn{3}{l}{\textit{Prior}} \\
    \midrule
    Position prior & \multicolumn{2}{c}{Anisotropic Gaussian (PCA-matched)} \\
    Atom-type prior & \multicolumn{2}{c}{$\text{Dir}(1/K, \dots, 1/K)$} \\
    \midrule
    \multicolumn{3}{c}{\textbf{Sampling}} \\
    \midrule
    \multicolumn{3}{l}{\textit{Mirror-Langevin Algorithm}} \\
    \midrule
    Step size $\eta$ & 0.05 & 0.1\\
    Coordinate noise $\sigma_\mathbf{c}$ &0.6 &  0.2 \\
    Atom-type noise $\sigma_\mathbf{p}$ & 0.6&  0.4 \\
    Simplex floor $\epsilon$ & 0.0015 & 0.0005\\
    \midrule
    \multicolumn{3}{l}{\textit{Parallel tempering}} \\
    \midrule
    Number of chains $M$  & 8 & 8 \\
    Temperature levels  & 11 & 11 \\
    Min/Max Temperature & 0.05/1.0 & 0.05/1.0 \\
    Steps between swaps & 10 & 10 \\
    Swaps between harvesting & [8, 12, 16] & [8, 16 ,32, 64] \\
    Swap criterion & Atom Max & Atom Max \\
    Relaxation steps ($\tau=0$) & 50 & 200 \\
    Relaxation step size        & 0.01 & 0.01 \\
    \bottomrule
  \end{tabular}
\end{table}

\paragraph{Simplex boundary stabilization.}
\label{sec:boundary_stabilization}
The Mirror Langevin Algorithm requires the mirror potential $\Phi$ to
satisfy a self-concordance-like condition (A1)
of~\cite{zhangWassersteinControlMirror2020}: there exists $\kappa \geq 0$
such that
\begin{equation}
    \sqrt{2}\,\bigl\lVert [\nabla^2\Phi(\mathbf{x})]^{1/2} -
    [\nabla^2\Phi(\mathbf{x}')]^{1/2} \bigr\rVert_F
    \leq \kappa\, \lVert \nabla\Phi(\mathbf{x}) -
    \nabla\Phi(\mathbf{x}') \rVert_2.
\end{equation}
For the negative-entropy potential
$\Phi^\mathbf{p}(\mathbf{p}) = \sum_i p_i \log p_i$, the Hessian is
$\nabla^2\Phi^\mathbf{p} = \mathrm{diag}(1/p_i)$ and its square root is
$[\nabla^2\Phi^\mathbf{p}]^{1/2} = \mathrm{diag}(1/\sqrt{p_i})$. As $p_i \to 0$,
the left-hand side of (A1) grows as $p_i^{-1/2}$ while the right-hand
side grows as $|\log p_i|$, so $\kappa = \infty$ on the full simplex
interior, consistent with Table~1
of~\cite{zhangWassersteinControlMirror2020}.

By clamping $\mathbf{p} \geq \epsilon$ before each update, we restrict
the domain to
$\Delta^{K-1}_\epsilon = \{\mathbf{p} \in \Delta^{K-1} : p_i \geq
\epsilon\ \forall i\}$, a compact subset of the simplex interior on
which $\kappa$ is finite. The choice of $\epsilon$ involves a tradeoff:
too small and the logits $\log p_i$ approach numerical limits, causing
instability in the dual-space update; too large and the floor acts as a
persistent noise source that prevents atom types from fully committing
to a vertex, degrading sample quality. In effect, $\epsilon$ controls
the minimum categorical noise injected at each step via the
$1/\sqrt{p_i}$ scaling, establishing a noise floor below which the
sampler cannot refine atom-type assignments. This makes $\epsilon$ a
dataset-dependent hyperparameter that requires tuning: datasets with
fewer atom types tolerate smaller values, while larger type vocabularies
need careful balancing to avoid both numerical instability and
premature convergence. We report the values used for each benchmark in
Table~\ref{tab:hyperparameters}.

\section{Generated samples}
\label{sec: generated samples}

Figure~\ref{fig:shape_samples} shows random samples for GEOM-Drugs using shape-steering, while Figure~\ref{fig:random_samples} shows unconditional samples from the same trained {EBMol}.

\begin{figure}[h!]
    \centering
    \includegraphics[width=1\linewidth]{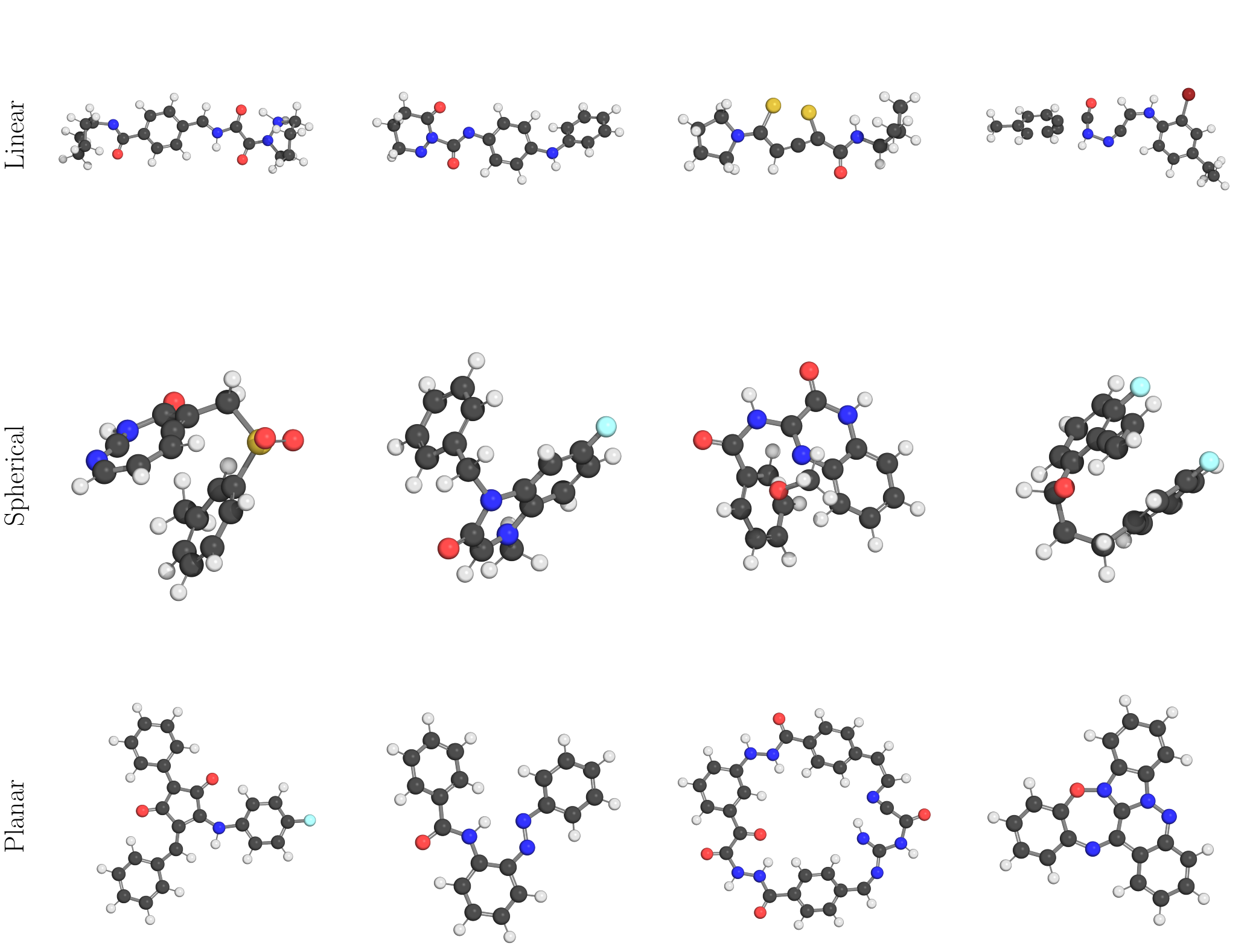}
  \caption{Random samples from an \textit{EBMol} model trained on GEOM-Drugs using potential composition to steer the molecule shape. Samples are generated with a compute budget of 1960 NFEs.}\label{fig:shape_samples}
\end{figure}

\begin{figure}[h!]
    \centering
    \includegraphics[width=1\linewidth]{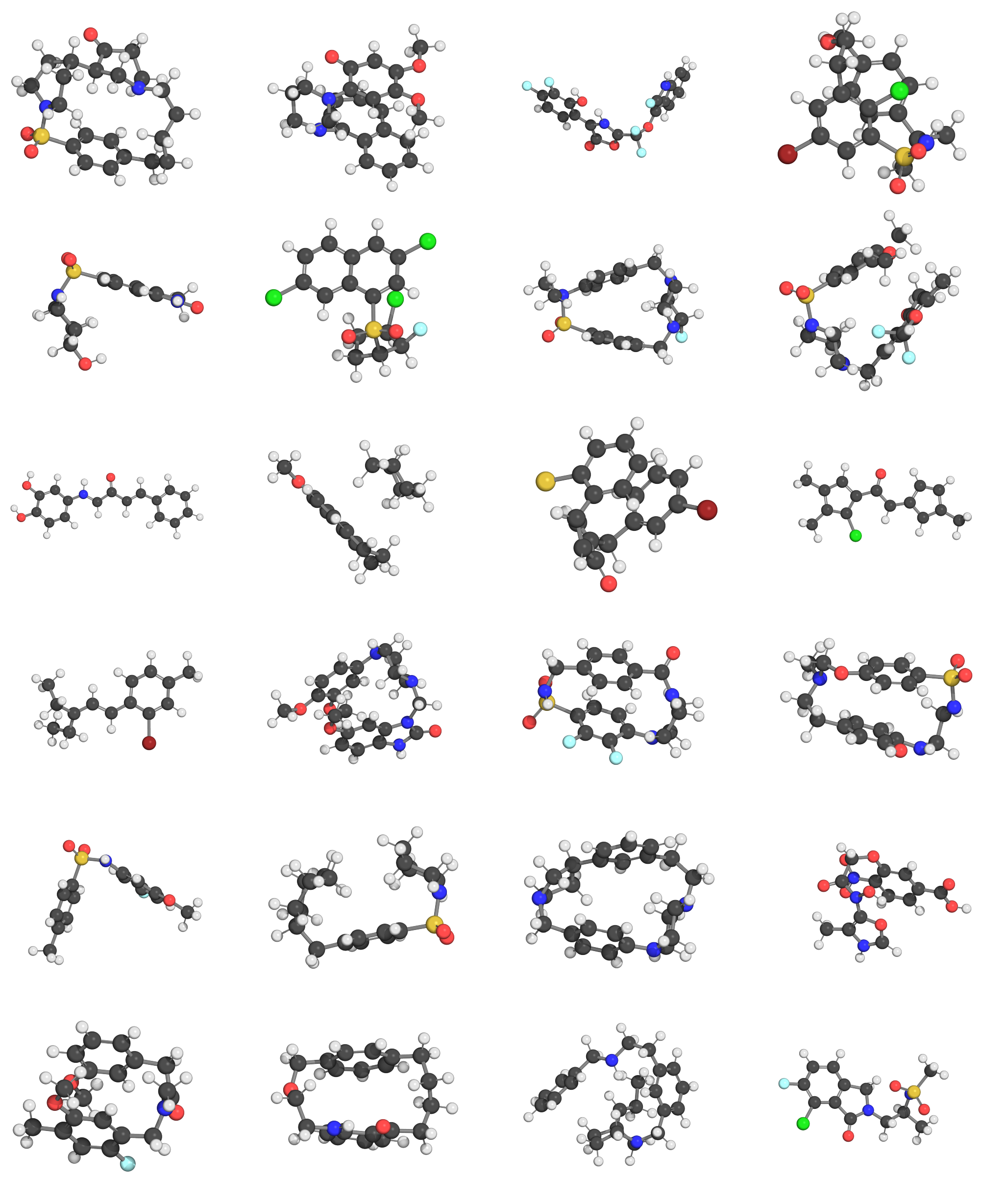}
  \caption{Random samples from an \textit{EBMol} model trained on GEOM-Drugs. Samples are generated with a compute budget of 1960 NFEs.}\label{fig:random_samples}
\end{figure}


\end{document}